%% file: iclr2021_conference.tex
\documentclass{article} 
\usepackage[dvipsnames]{xcolor}
\usepackage{iclr2021_conference,times}
\usepackage{graphicx}
\graphicspath{ {./figures/} }
\usepackage{enumitem}
\usepackage{nicefrac}

\input{math_commands.tex}

\usepackage{hyperref}
\usepackage{url}
\usepackage{amsthm}
\usepackage{stackengine}

\newcommand{\NF}{\textsf{NF}}
\newcommand{\mirNF}{\textsf{MNF}}
\newcommand{\mirN}{N^{\textsf{mir}}}
\newcommand{\MNF}{\textsf{MNF}}
\newcommand{\Glocal}{\mathcal{G}_{\text{local}}}

\newtheorem{theorem}{Theorem}
\newtheorem{lemma}{Lemma}[section]

\newtheorem{definition}{Definition}[section]
\newtheorem{fact}{Fact}[section]

\newcommand\numberthis{\addtocounter{equation}{1}\tag{\theequation}} 

\title{A Learning Theoretic Perspective on Local Explainability}

\author{Jeffrey Li\thanks{Denotes equal contribution} \\  University of Washington\\
\texttt{jwl2162@cs.washington.edu} \\
\And
Vaishnavh Nagarajan$^*$, Gregory Plumb \\
Carnegie Mellon University \\
\texttt{vaishnavh@cs.cmu.edu}
\\
\AND
Ameet Talwalkar \\
Carnegie Mellon University \& Determined AI \\
}

\usepackage[colorinlistoftodos]{todonotes}

\iclrfinalcopy 
\begin{document}

\maketitle

\begin{abstract} 

In this paper, we explore connections between interpretable machine learning and learning theory
through the lens of 
\textit{local approximation} explanations.
First, we tackle the traditional problem of \textit{performance generalization} and bound the test-time accuracy of a model using a notion of how locally explainable it is.  
Second, we explore the novel problem of \textit{explanation generalization} which is an important concern for a growing class of \textit{finite sample-based} local approximation explanations.
Finally, we validate our theoretical results empirically and show that they reflect what can be seen in practice. 

\end{abstract}

\section{Introduction}
\input{sections/introduction.tex}

\section{Related Work} 
\input{sections/related_work.tex}

\section{Mirrored Neighborhood Fidelity}
\label{sec:MNF}
\input{sections/mnf.tex}


\section{Generalization of Model Performance via $\mirNF$} 
\label{sec:performance}
\input{sections/performance_generalization.tex}

\section{Generalization of Local Explainability} 
\label{sec:explanations}
\input{sections/explainability_generalization.tex}


\section{Empirical Results}
\label{sec:experiments}
\input{sections/experiments.tex}

\section{Conclusion and Future Work} 
In this work, we have studied two novel connections between learning theory and local explanations. We believe these results may be of use in guiding the following directions of future work: (1) developing new local explanation algorithms inspired by our theory and the metric of $\mirNF$; (2)  resolving caveats or otherwise strengthening the theory presented in this paper; and (3) exploring applications of our techniques beyond interpretability, such as the general problem of deep learning generalization or others that require reasoning about the complexity of randomized functions.

\section*{Acknowledgments}
This work was supported in part by DARPA FA875017C0141, the National Science Foundation grants IIS1705121 and IIS1838017, an Okawa Grant, a Google Faculty Award, an Amazon Web Services Award, a JP Morgan A.I. Research Faculty Award, and a Carnegie Bosch Institute Research Award. Vaishnavh Nagarajan was supported by a grant from the Bosch Center for AI. Any opinions, findings and conclusions or recommendations expressed in this material are those of the author(s) and do not necessarily reflect the views of DARPA, the National Science Foundation, or any other funding agency.

\bibliography{iclr2021_conference}
\bibliographystyle{iclr2021_conference}

\newpage
\appendix
\input{sections/appendix.tex}

\end{document}

%% file: math_commands.tex
\usepackage{amsmath,amsfonts,bm}

\def\eqref#1{equation~\ref{#1}}

\def\1{\bm{1}}

\DeclareMathAlphabet{\mathsfit}{\encodingdefault}{\sfdefault}{m}{sl}
\SetMathAlphabet{\mathsfit}{bold}{\encodingdefault}{\sfdefault}{bx}{n}

\DeclareMathOperator*{\argmin}{arg\,min}

%% file: sections/introduction.tex
There has been a growing interest in interpretable machine learning, which seeks to help people understand their models. While interpretable machine learning encompasses a wide range of problems, it is a fairly uncontroversial hypothesis that there exists a trade-off between a \textit{model's complexity} and \textit{general notions of interpretability}. This hypothesis suggests a seemingly natural connection to the field of learning theory, which has thoroughly explored relationships between a \textit{function class's complexity} and \textit{generalization}. However, formal connections between interpretability and learning theory remain relatively unstudied.   

Though there are several notions of conveying interpretability, one common and flexible approach is to use local approximations. 
Formally, \textit{local approximation explanations} (which we will refer to as ``local explanations'') provide insight into a model's behavior as follows: for any black-box model $f \in \mathcal{F}$ and input $x$, the explanation system produces a simple function, $g_x(x') \in \Glocal$, which approximates $f$ in a specified neighborhood, $x' \sim N_x$. 
Crucially, the freedom to specify both $\Glocal$ and $N_x$ grants local explanations great versatility. 
In this paper, we provide two connections between learning theory and how well $f$ can be approximated locally (i.e. the \textit{fidelity} of local explanations).

Our first result studies the standard problem of \textit{performance generalization} by relating test-time performance to a notion of local interpretability. As it turns out, our focus on local explanations leads us to unique tools and insights from a learning theory point of view.
Our second result identifies and addresses an unstudied -- yet important -- question regarding \textit{explanation generalization}.
This question pertains to a growing class of explanation systems, such as MAPLE \citep{Plumb2018MAPLE} and RL-LIM \citep{Yoon2019}, which we call \textit{finite sample-based} local explanations\footnote{%
This terminology is not to be confused with example-based explanations where the explanation itself is in the form of data instances rather than a function.}.  
These methods learn their local approximations using a common finite sample drawn from $D$ (in contrast to local approximation methods such as LIME \citep{Ribeiro2016}) and, as a result, run the risk of overfitting to this finite sample.
In light of this, we answer the following question: for these explanation-learning systems, how well do they generalize to data not seen during training?

We address these questions with two bounds, which we outline now. 
Regarding performance generalization, we derive our first main result, Theorem \ref{thm:performance-generalization}, which bounds the expected test mean squared error (MSE) of any $f$ in terms of its MSE over the $m$ samples in the training set, $S=\{(x_i,y_i)\}_{i=1}^{m}$:
\begin{center}
\scalebox{0.90}{$
\underbrace{\vphantom{\frac{1}{m}\sum_{i=1}^m (f(x_i)-y_i)^2}
\mathbb{E}_{(x,y)\sim D}[(f(x)-y)^2]}_{\text{Test MSE}
} \leq  
\tilde{\mathcal{O}}\Big( \underbrace{\frac{1}{m}\sum_{i=1}^m (f(x_i)-y_i)^2}_{ \text{Train MSE}} + 
\underbrace{\vphantom{\frac{1}{m}\sum_{i=1}^m (f(x_i)-y_i)^2}
\mathbb{E}_{\stackunder[0pt]{$\scriptstyle{x\sim D},$}{$\scriptstyle{x' \sim N_{x}}$}}\left[ (g_{x'}(x)-f(x))^2 \right]}_{\text{Interpretability Term ($\mirNF$)}} 
+ 
\underbrace{\vphantom{\frac{1}{m}\sum_{i=1}^m (f(x_i)-y_i)^2}
\rho_S \hat{\mathcal{R}}_S(\Glocal)}_{\text{Complexity Term}} \Big)
$
}
\end{center}

Regarding explanation generalization for finite sample-based explanation-learning systems, we apply a similar proof technique to obtain Theorem 2, which bounds the quality of the system's explanations on unseen data in terms of their quality on the data on which the system was trained:
\begin{center}
\scalebox{0.9}{
$
\underbrace{\vphantom{\sum_{i=1}^{m} \mathbb{E}_{x' \sim N_{x}} \left[ (f(x_i)-g_{x'}(x_i))^2\right]}
\mathbb{E}_{\stackunder[0pt]{$\scriptstyle{x\sim D},$}{$\scriptstyle{x' \sim N_{x}}$}}\left[ (g_{x'}(x)-f(x))^2 \right]}_{\text{Test } \mirNF } 
\leq 
\underbrace{\frac{1}{m}\sum_{i=1}^{m}  \mathbb{E}_{x' \sim N_{x_i}} \left[ (f(x_i)-g_{x'}(x_i))^2\right]}_{ \text{Train } \mirNF} + 
\underbrace{\vphantom{\sum_{i=1}^{m} \mathbb{E}_{x' \sim N_{x}} \left[ (f(x_i)-g_{x'}(x_i))^2\right]}
\tilde{\mathcal{O}} \left (\rho_S \hat{\mathcal{R}}_S(\Glocal)\right)}_{\text{Complexity Term}}
$
}
\end{center}

Before summarizing our contributions, we discuss the key new terms and their relationship.
\begin{itemize}[leftmargin=0.2in]
    \item \textit{Interpretability terms:} 
    The terms involving $\mirNF$ correspond to Mirrored Neighborhood Fidelity, a metric we use to measure local explanation quality.
    As we discuss in Section \ref{sec:MNF}, this is a reasonable modification of the commonly used Neighborhood Fidelity ($\NF$) metric \citep{Ribeiro2016, Plumb2018MAPLE}.  
    Intuitively, we generally expect $\mirNF$ to be larger when the neighborhood sizes are larger since the $g_{x'}$ are required to extrapolate farther.
    \item \textit{Complexity term:} 
    This term measures the complexity of the local explanation system $g$ in terms of (a) the complexity of the local explanation class $\Glocal$ and (b) $\rho_S$, a quantity that we define and refer to as \emph{neighborhood disjointedness factor}. 
    As we discuss in Section \ref{sec:performance}, $\rho_S$ is a value in $[1,\sqrt{m}]$ (where $m=|S|$) that is proportional to the level of disjointedness of the neighborhoods for points in the sample $S$. 
    Intuitively, we expect $\rho_S$ to be larger when the neighborhoods sizes are smaller since smaller neighborhoods will overlap less. 
\end{itemize}

Notably, both our bounds capture the following key trade-off: as neighborhood widths increase, $\mirNF$  increases but $\rho_S$ decreases. As such, our bounds are non-trivial only if $N_x$ can be chosen such that $\mirNF$ remains small but $\rho_S$ grows slower than  $\tilde{\mathcal{O}}(\sqrt{m})$ (since  $\hat{\mathcal{R}}_S(\Glocal)$ decays as  $\tilde{\mathcal{O}}(1/\sqrt{m})$). 


We summarize our main contributions as follows:
\begin{itemize}[leftmargin=0.2in]

\item[(1)] 
We make a novel connection between performance generalization and local explainability, arriving at Theorem 1. 
Given the relationship between $\mirNF$ and $\rho_S$, this bound roughly captures that an easier-to-interpret $f$ enjoys better generalization guarantees, a potentially valuable result when reasoning about $\mathcal{F}$ is difficult (e.g. for neural networks).
Further, our proof technique may be of independent theoretical interest as it provides a new way to bound the Rademacher complexity of a \textit{randomized} function
(see Section \ref{sec:performance}).

\item[(2)] 
We motivate and explore an important generalization question about expected explanation quality. 
Specifically, we arrive at Theorem \ref{thm:explanation-generalization}, a bound for test $\mirNF$ in terms of training $\mirNF$.
This bound suggests that practitioners can better guarantee good local explanation quality (measured by $\mirNF$) using methods which encourage the neighborhood widths to be wider (see Section \ref{sec:explanations}).

\item[(3)]  
We verify empirically on UCI Regression datasets that our results non-trivially reflect the two types of generalization in practice. First, we demonstrate that $\rho$ can indeed exhibit slower than $\tilde{\mathcal{O}}(\sqrt{m})$ growth without significantly increasing the $\mirNF$ terms.
Also, for Theorem \ref{thm:explanation-generalization}, we show that the generalization gap indeed improves with larger neighborhoods (see Section \ref{sec:experiments}).

\item[(4)] 
To aid in our theoretical results, we propose $\mirNF$ as a novel measure of local explainability. 
This metric naturally complements $\NF$ and 
offers a unique advantage over $\NF$ when evaluating local explanations on ``realistic'' on-distribution data (see Section \ref{sec:MNF}).  

\end{itemize}

%% file: sections/related_work.tex
\textbf{Interpretability meets learning theory:}
\cite{Semenova2019Rashomon} study the performance generalization of models learned from complex classes when they can be globally well-approximated by simpler (e.g. interpretable) classes. In such cases, their theory argues that if the complex class has many models that perform about as optimally on training data, generalization from the complex class can be more closely bounded in terms of the simpler class's complexity. In our corresponding results, we similarly aim to avoid involving the larger class's complexity. However, we directly study generalization via a function's local explainability, rather than instantiate abstract "complex" and “simple” classes for global approximations. The two are fundamentally different technical problems; standard learning theory results cannot be directly applied as they are for single-function global approximations.


\textbf{Statistical localized regression:}
\citep{Fan1993, FanGijbels1996} are canonical results which bound the squared error of a nonparametric function defined using locally fit models. These local models are both simple (e.g. linear) and similarly trained by weighting real examples with a kernel (i.e. neighborhood). However, in these works, each local model is only used to make a prediction at its source point and the theory requires shrinking the kernel width towards 0 as the sample size grows.
We instead fit local models \textit{as explanations} for  a trained model (i.e. which is considered the ``true regression function'') and more importantly, care about the performance of each local model over whole (non-zero) neighborhoods. Unlike localized regression, this allows us to use uniform convergence to bound test error with empirical and generalization terms. While the previous results do not have empirical terms, the learning rates are exponential in the number of samples. 

\textbf{Learning Theory:} One line of related work also studies how to explain generalization of overparameterized classes. As standard uniform convergence on these classes often leads to vacuous bounds, a general approach that has followed from \citep{Nagarajan2019, Zhang2017, Neyshabur2014} has been to study implications of different biases placed on the learned models. We study what would happen if an overparameterized model had an unexplored type of bias, one that is inspired by local explainability. Additionally, our work's technical approach also parallels another line of existing results which likewise try to apply uniform convergence on a separate surrogate class. This includes PAC-Bayesian bounds, a large family of techniques that come from looking at a stochastic version of in parameter space \citep{McAllester1998, McAllester2003, Langford2002, Langford2003}. In a different vein, some results in deep learning look at compressed/sparsified/explicitly regularized surrogates of neural networks \citep{Arora2018, Dziugaite2017}. In our case, the surrogate class is a collection of local explanations. 

%% file: sections/mnf.tex
In order to connect local explanations to generalization, recall that we study a measure of local interpretability which we call ``mirrored neighborhood fidelity'' ($\mirNF$). As we explain below, this quantity comes from a slight modification to an existing measure of interpretability, namely, that of neighborhood fidelity ($\NF$).

To define our terms, we use the following notations. Let $\mathcal{X}$ be an input space and let $D$ be a distribution over $\mathcal{X} \times \mathcal{Y}$ where $\mathcal{Y} \subseteq \mathbb{R}$. Let $\mathcal{F}$ be a class of functions  $f: \mathcal{X} \to \mathcal{Y}$. For our theoretical results, we specifically assume that $\mathcal{Y}$ is bounded as $\mathcal{Y} = [-B,B]$ for some $B > 0$ (though this does not matter until following sections). 
 In order to provide local explanations, we need to fix a local neighborhood around each $x \in \mathcal{X}$. To this end, for any $x$, let $N_x$ correspond to some distribution denoting a local neighborhood at $x$ e.g., typically this is chosen to be a distribution centered at $x$. For any distribution $N$, we use $p_{N}(x)$ to denote its density at $x$. Now, let
 $\mathcal{G}$ be a class of explainers 
 $g: \mathcal{X} \times \mathcal{X} \to \mathcal{Y}$ such that for each $x \in \mathcal{X}$, the local explanation $g(x,\cdot): \mathcal{X} \to  \mathcal{Y}$ belongs to a class of (simple) functions (e.g. linear), $\Glocal$. In short, we denote $g(x, \cdot)$ as $g_x(\cdot)$ and we'll use 
  $g(x, \cdot)$ to locally approximate $f$ in the neighborhood defined by $N_x$. 
 
 The accuracy of this local approximation is usually quantified by a term called ``neighbhorhood fidelity'' which is defined as follows \citep{Ribeiro2016, Ribeiro2018, Plumb2018MAPLE, Plumb2020ExpO}
 \[
 \NF(f,g) := \mathbb{E}_{x \sim D} \left[ \mathbb{E}_{x' \sim N_x} \left[ (f(x') - g_x(x'))^2\right] \right].
 \]
To verbally interpet this, let us call $x$ as the ``source'' point which gives rise to a
 local explanation $g_x(\cdot)$ and $x'$ the ``target'' point that we try to fit using $g$. To compute $\NF(f,g)$, we need to do the following: for each source point $x$, we first compute the average error in the fit of $g_x(\cdot)$ over target points $x'$ in the local neighborhood of the source point $x$ (i.e., $N_x$); then, we globally average this error across draws of the source point $x \sim D$.
 
 Now, to define $\mirNF$, we take the same expression as $\NF$ but swap $x$ and $x'$ within the innermost expectation (without modifying the expectations). 
 In other words, 
 we now sample {\em a target point} $x$ from 
 $D$, and sample  {\em a source point} $x'$ from a distribution over points near $x$.
 Since this distribution is over source points rather than target points, just for the sake of distinguishing, we'll call this a {\em mirrored} neighborhood distribution 
 and denote it as $\mirN_x$.
 Formally we define this measure of local interpretability below, following which we explain how to understand it:

\begin{definition} \textbf{(Mirrored Neighborhood Fidelity)}
We define $\mirNF: \mathcal{F} \times \mathcal{G} \to \mathbb{R}$ as \[\mirNF(f,g) := \mathbb{E}_{x \sim D} \left[ \mathbb{E}_{x' \sim \mirN_x} \left[ (f(x) - g_{x'}(x))^2\right] \right].\] and with an abuse of notation, we let $\mirNF(f,g,x) := \mathbb{E}_{x' \sim \mirN_x} \left[ (f(x) - g_{x'}(x))^2\right]$.
\end{definition}

\textbf{Understanding $\mirNF$.} It is helpful to parse the expression for $\mirNF$ in two different ways. 
First, we can think of it as measuring the error in approximating every target point $x \in \mathcal{X}$ through a \textit{randomized} 
locally-approximating function $g_{x'}(\cdot)$ where $x'$ is randomly drawn from the local neighborhood $\mirN_{x}$.
A second way to parse this is in a manner similar to how we parsed $\NF$. To do this, first we note that the expectations in $\mirNF$ can be swapped around and rewritten equivalently as follows:
 \[
 \mirNF(f,g) = \mathbb{E}_{x' \sim D^\dagger} \left[ \mathbb{E}_{x \sim N^\dagger_{x'}} \left[ (f(x) - g_{x'}(x))^2\right] \right],
 \]
 where $D^{\dagger}$ and $N^\dagger_{x'}$ are suitably defined distributions (defined in Appendix~\ref{sec:mnf-appendix}) that can be thought of as modified counterparts of $D$ and $\mirN_{x'}$ respectively.  With this rewritten expression, one can read $\mirNF$ like $\NF$: for each source point (here that is $x'$), we compute the average error in the fit of the corresponding local function ($g_{x'}(\cdot)$) over target points ($x$) in the local neighborhood of the source point ($N^\dagger_{x'}$); this error is then globally averaged over different values of the source point ($x' \sim D^\dagger$).
 

While both $\NF$ and $\mirNF$ are closely related measures of local interpretability for $f$, studying $\mirNF$ allows us to make connections between local interpretability and different notions of generalization (Sections \ref{sec:performance} and \ref{sec:explanations}). Furthermore, $\mirNF$ may 
also be of interest to the interpretability community, as it offers a unique advantage over $\NF$ when the intended usage of local explanations is centered around understanding how the model works on the specific learning task it was trained on. 

Specifically, we argue that selecting the target point distribution to be $D$ rather than $D$ perturbed by $N_x$ (as for $\NF$) better emphasizes the ability for explanations to accurately convey how well $g$ will predict at \textit{realistic} points. This is relevant for ML (and deep learning particularly) because (a) high-dimensional datasets often exhibit significant feature dependencies and adherence to lower dimensional manifolds; (b) $f$ can often be highly unpredictable and unstable when extrapolating beyond the training data. As such, when one measures $\NF$ with standard neighborhood choices that ignore feature dependencies (i.e. most commonly $N_x = \mathcal{N}(x,\sigma I)$), the resulting target distribution may concentrate significantly on regions that are non-relevant to the actual task at hand. As can be shown, this can lead to overemphasis on fitting noisy off-manifold behavior, deteriorating the fit of explanations relative to task-relevant input regions (we defer a more detailed presentation of this point, as well as other trade-offs between $\NF$ and $\mirNF$ to Appendix \ref{sec:mnf-appendix}). 

%% file: sections/performance_generalization.tex
 The generalization error of the function $f$ is typically bounded by some notion of the representational capacity/complexity of $f$. While standard results bound complexity in terms of parameter counts, there is theoretical value in deriving bounds involving other novel terms. By doing so, we can understand how regularizing for those terms can affect the representation capacity, and in turn, the generalization error of $f$. Especially when $f$'s complexity may be intractable to bound on its own, introducing these terms provides a potentially useful new way to understand $f$'s generalization. 
 
 Here specifically, we are interested in establishing a general connection between the representation complexity and the local intrepretability of \textit{any} $f$. 
 This naturally requires coming up with a notion that appropriately quantifies the complexity of $\mathcal{G}$, which we discuss in the first part of this section. In the second part, we then relate this quantity to the generalization of $f$ to derive our first main result.   
 

 \textbf{Key technical challenge: bounding the complexity of $\mathcal{G}$.}
 The overall idea behind how one could tie the notions of generalization and local interpretability is fairly intuitive.
 For example, consider a simplified setting where we approximate $f$ by dividing $\mathcal{X}$ into $K$ disjoint pieces/neighborhoods, and then
 approximating each neighborhood via a simple (say, linear) model.
 Then, one could bound the generalization error of $f$ as the sum of two quantities: first, the error in approximating $f$ via the piecewise linear model,
 and second, a term involving the complexity of the piecewise linear model. It is straightforward to show that this complexity grows polynomially with the piece-count, $K$, and also the complexity of the simple local approximator (see Appendix \ref{sec:piecewise}). Similarly, one could hope to bound  the generalization error of $f$ in terms of $\mirNF(f,g)$ and the complexity of $\mathcal{G}$. However, the key challenge here is that the class $\mathcal{G}$ is a much more complex class than the above class of piecewise linear models. For example, a straightforward piece-count-based complexity bound would be infinitely large since there are effectively infinitely many unique pieces in $g$.



Our core technical contribution here is to bound the Rademacher complexity of $\mathcal{G}$ in this more complex local-interpretability setting.
At a high level, the resulting bound (which will be stated shortly) grows with ``the level of overlap'' between the neighborhoods $\{\mirN_x | x \in \mathcal{X} \}$, quantified as: 
\begin{definition}
Given a dataset $S \in (\mathcal{X} \times \mathcal{Y})^m$, we define the \textbf{disjointedness factor} $\rho_S$ as 
\[\rho_S := \int_{x' \in \mathcal{X}} \sqrt{\frac{1}{m}\sum_{i=1}^{m}(p_{\mirN_{x_i}}(x'))^2} dx'\]
\end{definition}
\textbf{Understanding the disjointedness factor.}  
$\rho_S$ can be interpreted as bounding the ``effective number'' of pieces induced by the set of neighborhood distributions $\{ \mirN_x | x \in \mathcal{X}\}$. This turns out to be a quantity that lies in $[1,\sqrt{m}]$ (shown formally in Appendix Fact~\ref{fact:rho-bound}).
To intuit about this quantity, it is helpful to consider its behavior in
 extreme scenarios. First, consider the case where $\mirN_x$ is the same distribution (say $N$) regardless of $x$
 i.e., neighborhoods are completely overlapping. Then, $\rho_S = \int_{x' \in \mathcal{X}} (p_{N}(x')) dx' = 1$. In the other extreme, consider if neighborhoods centered on the training data are all disjoint with supports $\mathcal{X}_1, \hdots, \mathcal{X}_{|S|}$. Here the integral splits into $m$ summands as: $\rho_S =  {\sum_{i=1}^{m}\int_{x' \in \mathcal{X}_i}  \frac{1}{\sqrt{m}}{p_{\mirN_{x_i}}(x')dx'}}  = \sqrt{m}$. Thus, intuitively $\rho_S$ grows from $1$ to $\sqrt{m}$ as the level of overlap between the neighborhoods $\smash{\mirN_{x_1}, \hdots, \mirN_{x_{|S|}}}$ reduces. For intuition at non-extreme values, we show in Appendix~\ref{sec:in-between} that in a simple setting,  $\rho = \sqrt{m^{{1-k}}}$ (where $0 \leq k \leq 1$) if every neighborhood is just large enough to encompass a $\nicefrac{1}{m^{1-k}}$ fraction of mass of the distribution $D$. 

\textbf{Rademacher complexity of $\mathcal{G}$.}
We now use $\rho_S$ to bound the Rademacher complexity of $\mathcal{G}$. First, in order to define the complexity of $\mathcal{G}$, it is useful to think of  $g$ as {\em a randomized function}. Specifically, at any target point $x$, the output of $g$ is a random variable $g_{x'}(x)$ where the randomness comes from $x' \sim \mirN_x$. 
Then, in Lemma~\ref{lem:rademacher}, we take this randomization into account to define and bound the complexity of $\mathcal{G}$ (which we use prove our main results). To keep our statement general, we consider a generic loss function $L: \mathbb{R} \times  \mathbb{R} \to \mathbb{R}$ (e.g., the squared error loss is $\smash{L(y,y')=(y-y')^2}$). Whenever $L$ satisfies a standard Lipschitz assumption, we can bound the complexity of $\mathcal{G}$ composed with the loss function $L$, in terms of $\rho_S$, the complexity of $\Glocal$ and the Lipschitzness of $L$:

 \begin{lemma}
 \label{lem:rademacher} (see Appendix Lemma~\ref{lem:rademacher-full} for full, precise statement)
Let $L(\cdot,y')$ be a $c$-Lipschitz function w.r.t. $y'$ in that for all $y_1, y_2 \in [-B,B]$, $|L(y_1,y') - L(y_2, y')| \leq c |y_1-y_2|$. Then, the empirical Rademacher complexity of $\mathcal{G}$ under the loss function $L$ is defined and bounded as:
  \[
  \hat{\mathcal{R}}_S(L \circ \mathcal{G}) := \mathbb{E}_{\vec{\sigma}}\left[ \sup_{g \in {\mathcal{G}}} \frac{1}{m} \sum_{i}^{m} \sigma_i \mathbb{E}_{x' \sim \mirN_{x_i}}[L(g_{x'}(x_i),y_i)] \right] \leq  O\left(c \rho_S \hat{\mathcal{R}}_S(\Glocal) \cdot  \ln m \right).
  \]
 \end{lemma}

We note that the proof technique employed here may be of independent theoretical interest as it provides a novel way to bound the complexity of a randomized function. Although techniques like PAC-Bayes provide ways to do this, 
they do not apply here since the stochasticity in the function is of a different form.

 
 \textbf{Main result.} With the above key lemma in hand, we are now ready to prove our main result, which bounds the generalization error of $f$ in terms of the complexity of $\mathcal{G}$, thereby establishing a connection between model generalization and local interpretability. 
 
\begin{theorem}
\label{thm:performance-generalization} (see Appendix Theorem~\ref{thm:performance-generalization-full} for full, precise statement)
With probability over $1-\delta$ over the draws of $S = \{(x_1,y_1), \hdots, (x_m,y_m) \} \sim D^m$, for all $f \in \mathcal{F}$ and for all $g \in \mathcal{G}$, we have (ignoring $\ln 1/\delta$ factors):
\begin{align*}
     \mathbb{E}_{(x,y) \sim D}[(f(x)-y)^2] & \leq \frac{4}{m}\sum_{i=1}^{m} (f(x_i)-y_i)^2 + 2 \underbrace{\mathbb{E}_{x \sim D}[\mathbb{E}_{x' \sim \mirN_x} \left[ (f(x) - g_{x'}(x))^2\right]]}_{\mirNF(f,g)} \\
     & + \frac{4}{m}\sum_{i=1}^{m} \underbrace{\mathbb{E}_{x' \sim \mirN_x} \left[ (f(x_i) - g_{x'}(x_i))^2\right]}_{\mirNF(f,g,x_i)} +
     \mathcal{O}(B\rho_S \hat{\mathcal{R}}_S(\Glocal) \ln m), 
\end{align*}
and $\hat{\mathcal{R}}_S(\Glocal)$ is the empirical Rademacher complexity of $\Glocal$ defined as $\hat{\mathcal{R}}_S(\Glocal) := \mathbb{E}_{\vec{\sigma}}\left[\sup_{h \in \Glocal }\frac{1}{m}\sum_{i=1}^{m} \sigma_i h(x_i)\right]$ where $\vec{\sigma}$ is uniformly distributed over $\{-1,1 \}^m$.
\end{theorem}

This result decomposes the test error of $f$ into four quantities. The first quantity corresponds to the training error of $f$ on the training set $S$. The second and the third correspond to the mirrored neighborhood fidelity of $f$ with respect to $g$ (computed on test and training data respectively). The fourth and final quantity corresponds to a term that bounds the complexity of $\mathcal{G}$ in terms of the ``disjointedness factor'' and the complexity of the simpler function class $\Glocal$.

\textbf{Takeaway.} A key aspect of this bound is the trade-off that it captures with varying neighborhood widths. Consider shrinking the neighborhood widths to smaller and smaller values, in turn creating less and less overlap between the neighborhoods of the training data. Then, on the one hand, we'd observe that the complexity term (the fourth term on the R.H.S) increases. Specifically, since $\hat{\mathcal{R}}_S(\Glocal)$ typically scales as $O(1/\sqrt{m})$, as we go from the one extreme of full overlap to the other extreme of complete disjointedness, the complexity term would increase from $O(1/\sqrt{m})$ to $O(1)$ (eventually rendering the bound trivial). On the other hand,
as the widths decrease, the fidelity terms (the second and the third term) would likely \textit{decrease} -- this is because the simple functions in $\Glocal$ would find it easier and easier to approximate the shrinking neighborhoods.
 
This tradeoff is intuitive. A function $f$ that is hardly amenable to being fit by local explanations would require 
 extremely tiny neighborhoods for $\Glocal$ to locally approximate it (i.e. make the $\mirNF$ terms small). For example, in an extreme case, when the neighborhoods $\mirN_x$ are set be point masses at $x$, it is trivially easy to find $g_x(\cdot) \in \Glocal$ with no approximation error. Thus, the complexity term would be too large in this case, implying that a hard-to-interpret $f$ results in bad generalization. On the other hand, when $f$ is easy to interpret, then we'd expect it to be well-approximated by $\Glocal$ even with wider neighborhoods. This allows one to afford smaller values for \textit{both} the complexity and $\mirNF$ terms. In other words, an easy-to-interpret $f$ enjoys better generalization guarantees.

\textbf{Caveats.} Our bound has two limitations worth noting. First, for high-dimensional datasets (like image datasets), practical choices of $N_x$ can lead to almost no overlap between neighborhoods, thus rendering the bound trivial in practice. This potentially poor dimension-dependence is a caveat similarly shared by bounds for non-parametric local regression, whereby increasing $d$ results in an exponential increase in the required sample size \citep{Fan1993, FanGijbels1996}. 
Nevertheless, for low-dimensional datasets, we show in the experiments that for practical choices of the neighborhood distributions, there \textit{is} sufficient neighborhood overlap to achieve values of $\rho_S$ that are $o(\sqrt{m})$. 

A second caveat is that the second quantity,  $\mirNF(f,g)$, requires \textit{unlabeled} test data to be computed, which may be limiting if one is interested in numerically computing this bound in practice. It is however possible to get a bound without this dependence, although only on the test error of $g$ rather than $f$ (see Appendix Theorem~\ref{thm:performance-generalization-full-2}). Nevertheless, we believe that the above bound has theoretical value in how it establishes a connection between the interpretability of $f$ and its generalization. 


%% file: sections/explainability_generalization.tex
We now turn our attention to a more subtle kind of generalization that is both unstudied yet important. Typically, the way $g_{x'}$ is learned at any source point $x'$ is by fitting a finite set of points sampled near $x'$, with the hope that this fit generalizes to unseen, neighboring target points. Naturally, we would want to ask: how well do the explanations $g_{x'}$ themselves generalize in this sense?

The subtlety in this question is that it is not always a worthwhile question to ask. In particular, assume that we learn $g_{x'}$ by sampling a set $S_{x'}$ of nearby points from a Gaussian centered at $x'$, and that we care about the fit of $g_{x'}$ generalizing to the same Gaussian. Here, we have access to unlimited amounts of data from the known Gaussian distribution (and free labels using $f(\cdot)$), so we can be sure that with sufficiently large $S_{x'}$, $g_{x'}$ will fit to arbitrarily small error on local neighborhoods. Hence, the above generalization question is neither conceptually nor practically interesting here.

However, consider \textit{finite sample-based} local explanation settings like MAPLE \citep{Plumb2018MAPLE} and RL-LIM \citep{Yoon2019} where the training procedure is vastly different from this:  in these procedures, the goal is to learn local explanations $g_{x'}$ in a way that is sensitive to the local structure of the (unknown) underlying data distribution $D$. So, instead of fitting the $g_{x'}$ to samples drawn from an arbitrarily defined Gaussian distribution, here one first draws a finite sample $S$ from the underlying distribution $D$ (and then labels it using $f$).
Then, across all $x' \in \mathcal{X}$, one reuses a reweighted version of the \textit{same} dataset $S$ (typically, points $x$ in $S$ that are near $x'$ are weighted more) and then learns a  $g_{x'}$ that fits this reweighted dataset. Contrast this with the former setting, where for each $x'$, one has access to a \textit{fresh} dataset (namely, $S_{x'}$) to learn $g_{x'}$. This distinction then makes it interesting to wonder when the reuse of a common dataset $S$ 
could cause the explanations to generalize poorly.



Motivated by this question, we present Theorem~\ref{thm:explanation-generalization}. By using Lemma~\ref{lem:rademacher}, we provide a bound on the ``test MNF'' (which corresponds to the fit of $g_{x'}$ on the unseen data averaged across $D$) in terms of the ``train MNF'' (which corresponds to the fit of $g_{x'}$ on $S$, averaged across $x'$) and the complexity term from Lemma~\ref{lem:rademacher}.  We must however caution the reader that this theorem does \textit{not} answer the exact question posed in the above paragraph; it only addresses it indirectly as we discuss shortly.

\begin{theorem}
\label{thm:explanation-generalization} (see Appendix Theorem~\ref{thm:explanation-generalization}-full for full, precise statement)
For a fixed function $f$, with high probability $1-\delta$ over the draws of $S \sim D^m$, for all $g \in \mathcal{G}$, we have (ignoring $\ln 1/\delta$ factors):

\[
\underbrace{\mathbb{E}_{\stackunder[0pt]{$\scriptstyle{x\sim D},$}{$\scriptstyle{x' \sim N_{x}}$}} \left[ (f(x)-g_{x'}(x))^2\right]}_{\text{test } \mirNF \text{ i.e., } \mirNF(f,g)} \leq \underbrace{\frac{1}{m}\sum_{i=1}^{m} \mathbb{E}_{x' \sim \mirN_{x}} \left[ (f(x_i)-g_{x'}(x_i))^2\right]}_{ \text{train } \mirNF} + O(\rho_S \mathcal{R}_S(\Glocal) \ln m).
\]
\end{theorem}

\textbf{Understanding the overall bound.} We first elaborate on how this bound provides an (indirect) answer to our question about how well explanations generalize. 
Consider a procedure like MAPLE that learns $g$ using the finite dataset $S$. For each $x' \in \mathcal{X}$, we would expect this procedure to have learned a $g_{x'}$ that fits well on at least those target points $x$ in $S$ that are near $x'$. 
In doing so, it's reasonable to expect the training procedure to have implicitly controlled the ``train $\mirNF$'' term 
The reasoning for this is that the train $\mirNF$ computes the error in the fit of $g_{x'}$ on $S$ for different values of $x'$, and sums these up in a way that errors corresponding to nearby values of $(x,x')$ are weighted more (where the weight is given by $p_{\mirN_x}(x')$). 
 Now, our bound suggests that when this train $\mirNF$ is minimized, this carries over to test MNF too (provided the complexity term is not large). That is, we can say that the fit of $g_{x'}$ generalizes well to unseen, nearby target points $x$ that lie outside of $S$. 
 

\textbf{The indirectness of our result.} Existing finite sample-based explainers do not explicitly minimize the train $\mirNF$ term (e.g., MAPLE minimizes an error based upon $\NF$). However, as argued above, they have implicit control over train $\mirNF$. Hence, our bound essentially treats $\mirNF$ as a surrogate for reasoning about the generalization of the explanations learned by an arbitrary procedure. As such, our bound does \textit{not} comment on how well the exact kind of fidelity metric used during training generalizes to test data. Nevertheless, we hope that this result offers a concrete first step towards quantifying the generalization of explanations. Furthermore, we also note that one could also imagine a novel explanation-learning procedure that does explicitly minimize the train $\mirNF$ term to learn $g$; in such a case our bound would provide a direct answer to how well its explanations generalize. Indeed, we derive such a theoretically-principled algorithm in Appendix \ref{sec:mnf-appendix}.
 

\textbf{Takeaway.} 
While the above bound captures a similar trade-off with neighborhood width as the Theorem \ref{thm:performance-generalization}, it is worth pausing to appreciate the distinct manner in which this tradeoff arises here.  In particular, when the width is too small, we know that the complexity term approaches $O(\sqrt{m})$ and generalization is poor. Intuitively, this is because in this case, the procedure for learning $g_{x'}$ would have been trained to fit very few datapoints from $S$ that would have fallen in the small neigbhorhood of $x'$.
On the other hand, when the neighborhoods are large, this issue would not persist which is captured by the fact that $\rho_S$ approaches $O(1)$. 
However, with large neighborhoods, it may also be hard to find functions in $\Glocal$ that fit so many points in $S$.
Overall, one practical takeaway from this bound is that it is important to not excessively shrink the neighborhood widths if one wants explanations that generalize well for predicting how $f$ behaves at unseen points (see Section \ref{sec:experiments}). 

\textbf{Caveats.} We remark that this particular bound applies only when the dataset $S$ is used to learn only $g$ i.e., $f$ and the neighborhoods must be learned beforehand with separate data. This sort of a framework is typical when deriving theoretical results for models like random forests, where it greatly aids analysis to assume that the trees' splits and their decisions are learned from independent datasets (i.e. two halves of an original dataset) \citep{ArlotGenuer2014}. Now, if one is interested in a bound where $S$ is also used to simultaneously learn $f$, the only change to the bound is an added factor corresponding to the complexity of $\mathcal{F}$. 
Another caveat is that our bound only tells us how well the explanations $g_{x'}$ generalize \textit{on average} over different values of $x'$. This does not tell us anything about the quality of the generalization of $g_{x'}$ for an arbitrary value of $x'$. That being said, just as average accuracy remains a central metric for performance (despite ignoring discrepancies across inputs), average $\mirNF$ can still be a useful quantity for evaluating an explainer's overall performance.  


%% file: sections/experiments.tex
We present two sets of empirical results to illustrate the the usefulness of our bounds. 
First, we demonstrate that $\rho_S$ grows much smaller than $\mathcal{O}(\sqrt{m})$ which, as stated before,  establishes that our bounds yield meaningful convergence rates.
Second, we show that Theorem \ref{thm:explanation-generalization} accurately reflects the relationship between explanation generalization (Theorem \ref{thm:explanation-generalization}) and the width of $\mirN_x$ used to both generate and evaluate explanations. 

\textbf{Setup.} 
For both experiments, we use several regression datasets from the UCI collection \citep{Dua:2019} and standardize each feature to have mean 0 and variance 1. 
We train neural networks as our ``black-box'' models with the same setup as in \citep{Plumb2020ExpO}, using both their non-regularized and ExpO training procedures. 
The latter explicitly regularizes for $\NF$ during training, which we find also decreases $\mirNF$ on all datasets. 
For generating explanations, we define $\Glocal$ to be linear models and optimize each $g_x$ using the empirical $\mirNF$ minimizer (see Appendix \ref{sec:mnf-appendix}). 
Finally, we approximate $\rho_S$ using a provably accurate sampling-based approach (see Appendix \ref{experiments-appendix}).

\begin{figure}[h]
    \label{fig:experiments}
    \centering \includegraphics[width=0.3\textwidth]{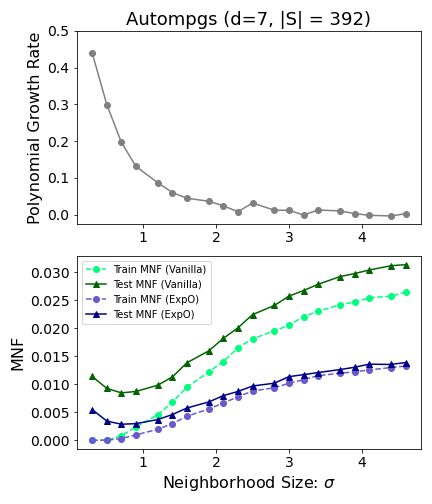} \includegraphics[width=0.3\textwidth]{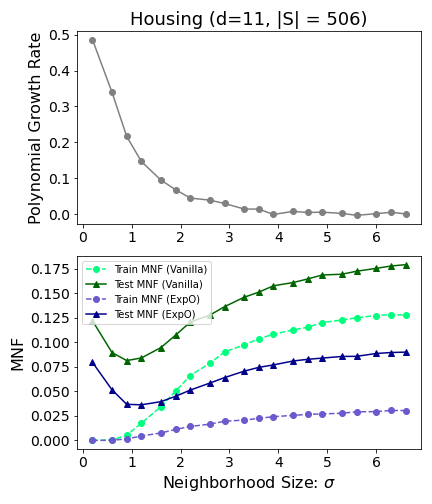} \includegraphics[width=0.3\textwidth]{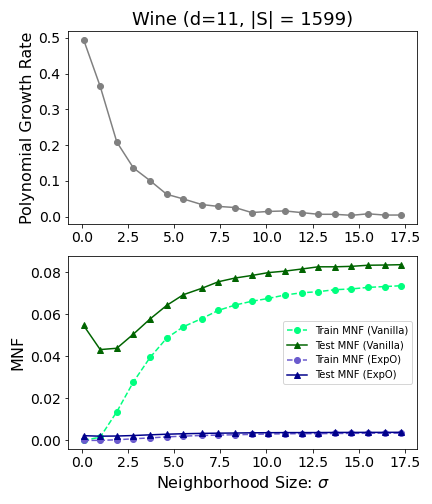} 
     \caption{ Approximate exponent of $\rho_S$'s polynomial growth rate (top) and train/test $\mirNF$ (below) plotted for various neighborhood widths across several UCI datasets (see Appendix \ref{experiments-appendix} for more). 
     }
    \label{fig:experiments}
\end{figure} 

\textbf{Growth-rate of $\rho_S$.} 
In Figure \ref{fig:experiments} (top), we track the sample  dependence of $\rho_S$ for various neighborhoods of width $\sigma$ (setting $\mirN_x = \mathcal{N}(x,\sigma I)$). We specifically approximate the growth rate as polynomial, estimating the exponent by taking the overall slope of a log-log plot of $\rho_S$ over $m$. To cover a natural range for each dataset, $\sigma$ is varied to be between the smallest and half the largest inter-example $l_2$ distances. In these plots, while small $\sigma$ result in a large exponent for $\rho_S$ and large $\sigma$ cause $g$ to intuitively saturate towards a global linear model, we observe that there do exist values of $\sigma$, where both these terms are in control i.e., we observe that we can achieve a growth rate of approximately $\mathcal{O}(m^{0.2})$ without causing $g$ to saturate and $\mirNF$ metrics to rise sharply. 

\textbf{Generalization and neighborhood size.} 
As per the setting of Theorem \ref{thm:explanation-generalization}, we generate all explanations using data not used to learn the black-box model. 
Specifically, we split the original test data into two halves, using only the first half for \textit{explanation training} and the second for \textit{explanation testing}. 
We plot $\mirNF$ as measured over these two subsets of examples in Figure \ref{fig:experiments} (bottom). 
From the results, it is evident that a generalization gap between train and test $\mirNF$ exists. 
Further, recall that Theorem \ref{thm:explanation-generalization} predicts that this gap decreases as wider neighborhoods are used, a phenomena reflected in most of these plots. 
As a result, while training $\mirNF$ monotonically increases with larger neighborhoods, test $\mirNF$ always decreases at certain ranges of $\sigma$.

%% file: sections/appendix.tex
\section{More on Mirrored neighborhood fidelity}
\label{sec:mnf-appendix}
Here we elaborate on how the expression for $\mirNF$ can be parsed in the same way as $\NF$ after  juggling some terms around. Recall that $\mirNF$ is defined as:
\[\mirNF(f,g) := \mathbb{E}_{x \sim D} \left[ \mathbb{E}_{x' \sim \mirN_x} \left[ (f(x) - g_{x'}(x))^2\right] \right].\] and with an abuse of notation, we let $\mirNF(f,g,x) = \mathbb{E}_{x' \sim \mirN_x} \left[ (f(x) - g_{x'}(x))^2\right]$.

Here the outer expectation is over the \textit{target} points $x$ that the explanations try to fit, and the inner expectation is over the \textit{source} points $x'$ which give rise to the explanations $g_{x'}$. 

If we can swap these expectations around, we can afford a similar parsing as $\NF$. To get there, first consider the joint distribution over $x$ and $x'$ that is induced by generating $x \sim D$ and then picking $x' \sim \mirN_x$. Under this joint distribution, 
 we need an expression for the marginal distribution of $x'$. This distribution, which we denote by $D^{\dagger}$, is given by:

\[
p_{D^\dagger}(x') = \int_{\mathcal{X}} p_{D}(x) p_{\mirN_{x}}(x') dx.
\]
To get a sense of what $D^{\dagger}$ looks like, imagine that $\mirN_x$ is a Gaussian centered at $x$. Then $D^{\dagger}$ corresponds to convolving $D$ with a Gaussian i.e., a smoother version of $D$.

Next, under the same joint distribution, we need an expression for the distribution of $x$ conditioned on $x'$. This distribution, denoted as $N^\dagger_{x'}$, is given by:

\[
p_{N^\dagger_{x'}}(x) = \frac{p_{D}(x)p_{\mirN_x}(x')}{\int_{\mathcal{X}} p_{D}(x) p_{\mirN_{x}}(x') dx}.
\]

Intuitively, ${N}_{x'}^{\dagger}$ is distribution that is centered around $x'$ and is also weighted by the distribution $D$ i.e., points that are both close to $x'$ and realistic under $D$ have greater weight under  ${N}_{x'}^{\dagger}$. This is because the term $p_{\mirN_{x}}(x')$ in the numerator prioritizes points that are near $x'$ (imagine $\mirN_x$ being a Gaussian centered at $x$), and the term $p_{D}(x)$ prioritizes realistic points. 

With these definitions in hand, we can swap the expectations around and get:
 \[
 \mirNF(f,g) = \mathbb{E}_{x' \sim D^\dagger} \left[ \mathbb{E}_{x \sim N^\dagger_{x'}} \left[ (f(x) - g_{x'}(x))^2\right] \right],
 \]
 This then has the same structure as $\NF$ in that the outer expectation is over the source points and the inner distribution over target points, and hence can be interpreted similarly. 

\subsection{Algorithm for minimizing empirical mirrored neighborhood fidelity}

We now consider how one might actually fit explanations to minimize $\mirNF$. Recall from the above discussion that from the point of view of each source point $x'$, $\mirNF$ measures how well $g_{x'}$ fits $f$ on the distribution with density $p_{N^\dagger_{x'}}(x) = \frac{p_{D}(x)p_{\mirN_x}(x')}{\int_{\mathcal{X}} p_{D}(x) p_{\mirN_{x}}(x') dx}$. Note that one does not have access to samples from this distribution due to the dependence on $D$. However as we argue, one can minimize the \textit{empirical} version of $\mirNF$ given access to a finite sample $S$ drawn i.i.d. from $D$ by solving the following \textit{weighted} regression problem:
$$ g_{x'} = \argmin_{g_{x'} \in \Glocal} \frac{1}{|S|}\sum_{i=1}^{|S|} (g_{x'}(x_i) - f(x_i))^2 p_{\mirN_{x_i}}(x')$$
\noindent To see what the empirical version of $\mirNF$ is, we can replace the outer expectation (over $x \sim D$) with the samples $S = \{x_i\}_{i=1}^{|S|}$, giving us 

\begin{align*} \text{Empirical $\MNF$} &= \frac{1}{|S|}\sum_{i=1}^{|S|} \mathbb{E}_{x' \sim \mirN_{x_i}} \left[ (g_{x'}(x_i) - f(x_i))^2 \right] \\
&= \frac{1}{|S|}\sum_{i=1}^{|S|} \int_{\mathcal{X}} (g_{x'}(x_i) - f(x_i))^2 p_{\mirN_{x_i}}(x') dx' \\ 
&= \frac{1}{|S|} \int_{\mathcal{X}} \sum_{i=1}^{|S|} (g_{x'}(x_i) - f(x_i))^2 p_{\mirN_{x_i}}(x') dx'
\end{align*}

To minimize the overall empirical $\mirNF$, one needs to choose $g_{x'}$ for each $x'$ such that it minimizes the above integrand, which is akin to performing one weighted least squares regression. Thus, one notable difference between minimizing empirical $\mirNF$ and $\NF$ is that we need to use real examples to fit $g_{x'}$ for $\mirNF$ but not for $\NF$ since the target distribution of interest there can be user-defined (i.e. it may be chosen such that it can be easily sampled from).

\subsection{Trade-offs between $\mirNF$ and $\NF$} 
We now discuss in further detail the comparison between $\mirNF$ and $\NF$, listing both the relative advantages and disadvantages of each.
It should be noted that this discussion is of a somewhat more exploratory nature; we do not aim to make definitive value judgments (i.e. one metric is always more useful than the other), but rather to provide a better qualitative understanding of how these two metrics might be expected to behave. 
We hope that this discussion prompts a more careful consideration of fidelity metrics in future works involving local explanations.

\subsubsection{Advantages of $\mirNF$}
In many practical situations (esp. for i.i.d. cases), it is reasonable to assume that practitioners will care significantly about generating explanations for predictions at \textit{realistic} on-distribution points and hoping that those (local) models correctly approximate what the model will do at nearby points which are also realistic. 
Our core argument for the usefulness of $\mirNF$ compared to $\NF$ is that it can be used to come closer to characterizing performance relative to the second part of this goal (i.e. predicting what the model will do at realistic points).

To reiterate Section \ref{sec:MNF}, this is an especially important concern for modern ML settings, which often involve significant feature dependencies (i.e. lower dimensional data manifolds) and models that behave unstably when extrapolating beyond the given task and training data. As we illustrate below in a toy example, when one uses $\NF$ with standard neighborhood choices (i.e. $N_x = \mathcal{N}(0,\sigma I)$), one may overemphasize the ability of explanations to fit this noisy behavior on regions that are off-manifold.

\textbf{Toy example.} We compare the abilities of $\mirNF$ and $\NF$ to serve as the basis for generating local explanations. In what follows, we refer to $g^{\NF}$ and $g^{\mirNF}$ as the explanations that minimize $\NF$ and $\mirNF$ respectively. We specifically consider a simple setup where the full input space has dimension $d=2$ but the data exists on a manifold of dimension $k=1$. Under task-distribution $D$, let $x_1 \sim \mathcal{N}(0,1)$ while $x_2 = 0$. Further consider the learned model $f(x) = x_1  - \beta x_1 x_2^2$, where one may assume $\beta \gg 0$. As an important note, on the task distribution $D$, $f(x) \equiv x_1$. 

\underline{Minimizing $\NF$:} To learn $g_x^{\NF}$, we may simply sample many $x' \sim N_x$ and find a linear $g^{\NF}_x(\cdot)$ that fits these points well. Now, we can expect this process to generalize in a way that $\mathbb{E}_{x' \sim N_x}[(g_x(x')-f(x'))^2]$ is minimized. In fact, one could consider the ideal scenario where we sample infinitely many unlabeled examples, and thus find the best possible linear approximation given this neighborhood distribution. However, observe that minimizing the above quantity provides absolutely \textit{no guarantee whatsoever} as far as the error committed on $D$ i.e., $\mathbb{E}_{x' \sim D}[(g_x(x')-f(x'))^2]$. This is because $D$ has zero measure. 
This means that by creating $f$ that is arbitrarily volatile along the irrelevant $x_2$ direction, we can force $g_x$ to be severely incorrect on $D$. Indeed, this is the case in the setting above. Let $g_x(x') = w_1 x_1' + w_2 x_2'$ and $N_x = N(0,I)$. Then, it can be shown that $\NF(f,g,x)$ is minimized by $w_1=1-\beta$. Since $\beta$ can be arbitrarily large, this explanation can be unboundedly arbitrarily poor at recovering a function equivalent to $f(x) \equiv x_1$ on $D$. 

\underline{Minimizing $\mirNF$:} Note that none of the above is a problem when we learn $g^{\mirNF}$, because we fit $g^{\mirNF}_x$ only on target points 
that are from the real data manifold. This will ensure that $g^{\mirNF}$ is in line with a potentially important desiderata for local explanations i.e., that they can faithfully capture a function that is accurate along the task-relevant data directions (of course, only upto a linear approximation). To illustrate more completely, recall that  $g^{\mirNF}$ is learned as follows: assuming access to 
$S = \{x_1, \dots, x_m\} \sim D^{m}$,
we have

$$ g_{x'}^{\mirNF} = \argmin_{g_{x'} \in \Glocal} \frac{1}{m}\sum_{i=1}^{m} (g_{x'}(x_i) - f(x_i))^2 p_{\mirN_{x_i}}(x')$$

Now since $S$ lies on the manifold of $D$, we have that $x_2=0$ on all those points. Therefore,  for each $x$,  we find the solution which minimizes 

\[
g_{x'}^{\mirNF} = \argmin_{w_1 \in \mathbb{R}} \frac{1}{m}\sum_{i=1}^{m}  (x_{i,1}-w_1x_{i,1} )^2p_{\mirN_{x_i}}(x')
\]

It is easy to see that with just two distinct datapoints from $S$, we would get $w_1 = 1$, which leads to perfect predictions for how the function behaves on $D$. 

As a remark, an even more natural version of the above setting would be one where $f$ is non-linear even on the data manifold. But even here we can still argue that $g^{\mirNF}$ would be close to the best possible linear function within the manifold up to a $1/\sqrt{m}$ error (e.g., a generlization bound like our Theorem \ref{thm:explanation-generalization} guarantees this on average over $x$). On the other hand, regardless of how many unlabeled datapoints we fit $g^{\NF}$ with, we would learn $g^{\NF}$ that can behave arbitrarily poorly on the manifold.


\subsubsection{Limitations of $\mirNF$ } Below, we discuss some limitations of $\mirNF$ as well as potential future directions for possibly addressing them. At a high-level, we believe while each represents a legitimate concern, they may arguably be (depending on context) ``reasonable prices to pay'' for the advantages of $\mirNF$ compared to $\NF$ described previously.

\textbf{$\mirNF$ explanations may lose local meaning:} Using $\mirNF$ to evaluate/generate explanations at low-probability source points $x'$ may have little to do with how $f$ actually behaves around $x'$. Because the target point distribution is  $x|x' \propto p_{D}(x)p_{\mathcal{N}_x}(x')$, very little probability mass might be placed in the vicinity around $x'$ when $p_{D}(x)$ is small. This would be the case when $x'$ is off-manifold or in low-density regions on the support of the real data distribution. The former might be dismissable if one cares about i.i.d. settings but the latter could be very important in applications where rare cases correspond to high-stakes decisions (e.g. disease diagnostics). In these scenarios, the explanation might still be too biased towards how the model is behaving at higher density regions. However, some potential future directions to remedy this are:
\begin{itemize}
    \item It might help to allow $\mirN_x$ to have smaller width around lower probability points from $D$ (allowing you to concentrate $\mirN_x$ around $x$ despite the form of $D$). It's remains a challenge to see how one would actually set these widths but it could be of help if a limit can be assumed on how quickly the value $p_{D}(x)$ can change around $x$.
    \item There also could be some use in considering a more general definition of $\mirNF$ that lets you choose an arbitrary outer distribution $x \sim \mathcal{Q}$ other than simply the task distribution $D$. That is, if one really cares about mutually consistent explanations in some arbitrary region (which could be on or off-manifold), then this would potentially allow one to able to customize a metric for that purpose. 
\end{itemize}

\textbf{Less intuitive target point neighborhoods: } Very closely related to the previous limitation, in interpreting $\mirNF$-based explanations, an end-user would have to understand that $g_{x'}$ are not exactly approximations for the locality around $x'$ but rather for the true target distribution that captures in some sense ``on-manifold points near $x'$ (modulated by the concentration of $\mirN_x$).'' This makes it harder for a user to know the exact window in which their explanation is directly valid for (compared to a user-specified target neighborhood for $\NF$). In practice, this shortcoming could be mitigated as long is it is carefully communicated to users that this limitation exists, i.e. they should focus on using $\mirNF$ explanations only at and for predicting what happens at realistic points. 
    
    \textbf{Unnaturalness of source points}: While $\mirNF$ does emphasize realistic target points, it also focuses on explanations generated at potentially off-manifold source points. Further, one could argue that the advantages of $\mirNF$ are partly because $N_x$ is chosen naively for $\NF$. For instance if one defined $N_x = N^{\dagger}_x$ in the definition for $\MNF$, then $g^{\NF}$ and $g^{\mirNF}$ would produce the same explanations because the inner target point expectations would be the same (comparing $\NF$ and the reversed expectation form of $\mirNF$). However, the average metric for $\NF$ seem more natural in an additional sense since it also only reflects caring about realistic \textit{source} points when looking at the outer expectation over $x \sim D$.
    
    $$\NF = \mathbb{E}_{x\sim D} \mathbb{E}_{x' \sim N^{\dagger}_x} \left[ [(g_x(x') - f(x')]^2\right]$$
    $$\mirNF = \mathbb{E}_{x'\sim D^{\dagger}} \mathbb{E}_{x \sim N^{\dagger}_{x'}} \left[ [(g_{x'}(x) - f(x)]^2\right]$$
    
    Given this, might $\mirNF$ be less interesting on its own?
    Using standard  ``naive" settings of $N_x$, one could argue that $\NF$ is also ``unnatural'' in that it takes into account how explanations at on-manifold source points perform at off-manifold target points. And though the above $\NF$ setting may be more ideal as a metric, it also becomes less clear how to evaluate it as the inner distribution cannot be sampled from easily. 
    On the other hand, we can use the original form of writing out $\mirNF$  (without the expecations flipped) to directly approximate $\mirNF$ with relevant samples from $D$.
    
    \textbf{Does not reflect what model causally depends on: } In the second toy-example, it was shown that if $f(\mathbf{x}) = x_1 - \beta x_1 x_2^2$ but the data manifold is $(x_1, x_2) = (x_1, 0)$, one could get arbitrarily poor fidelity and feature relevancy (for $x_1$) on this manifold using standard neighborhoods. But $\mirNF$ runs into a new problem when the feature set actually includes a highly correlated third feature: for example, consider $(x_1, x_2, x_3)$ where the manifold is defined by points $(x_1, x_2, x_3) = (x_1, 0, x_1)$. Thus according to $\mirNF$, $g(\mathbf{x}) = x_1$,  $g(\mathbf{x}) = x_3$, and indeed $g(\mathbf{x}) = -x_1 + 2x_3$ are all equally good explanations (because $\mirNF$ only cares about whether $g(\mathbf{x}) = f(\mathbf{x})$ on manifold). However, $f$ clearly only ``depends'' on $x_1$ for its decisions (in a causal sense). On the other hand, because $\NF$ samples target points both on and off manifold, it would correctly see that $x_3$ has no effect. The larger argument here is that in any conversation involving manifolds, one inherently is speaking about some sort of feature dependencies, which may similarly suffer from the same issues of not being causal w.r.t. $f$ and having non-identifiable explanations. On the other hand, we note that in the new toy-example, $\NF$ is not an ideal fix either because the cost is potentially an arbitrary coefficient for $x_1$ and extremely poor fidelity on $D$. More generally, finding ``what the model uses for its decision'' is simply not what $\mirNF$ explanations are trying to do. What one could describe $\mirNF$ as actually looking at is ``can I build a simpler local model relevant to the actual task at hand?''

\newpage \section{More on the disjointedness factor}

\subsection{Bounds}
Recall that the disjointedness factor is defined as $\rho_S := \int_{x' \in \mathcal{X}} \sqrt{ \frac{\sum_{j=1}^{m} (p_{\mirN_{x_i}}(x'))^2}{m}} dx'$. Here, we show that the disjointedness factor is bounded between $1$ and $\sqrt{m}$.

\begin{fact}
\label{fact:rho-bound}
The disjointedness factor $\rho_S$ satisfies $1\leq \rho_S \leq m$.
\end{fact}

\begin{proof}
For the lower bound, we note that since the arithmetic mean lower bounds the quadratic mean, we have:

\begin{align*}
    \int_{x' \in \mathcal{X}} \sqrt{ \frac{\sum_{j=1}^{m} (p_{\mirN_{x_i}}(x'))^2}{m}} dx' &\geq
    \int_{x' \in \mathcal{X}}  \frac{\sum_{j=1}^{m} p_{\mirN_{x_i}}(x')}{m} dx' \\
    & \geq \sum_{j=1}^{m} \frac{1}{m}\int_{x' \in \mathcal{X}}  p_{\mirN_{x_i}}(x') dx' \\
    &  \geq \sum_{j=1}^{m} \frac{1}{m} = 1 \\
\end{align*}

For the upper bound, we make use of the fact that the $\ell_2$ norm of a vector is smaller than its $\ell_1$ norm to get:

\begin{align*}
    \int_{x' \in \mathcal{X}} \sqrt{ \frac{\sum_{j=1}^{m} (p_{\mirN_{x_i}}(x'))^2}{m}} dx' &\leq
    \int_{x' \in \mathcal{X}}  \frac{\sum_{j=1}^{m} p_{\mirN_{x_i}}(x')}{\sqrt{m}} dx' \\
    & \leq \sum_{j=1}^{m} \frac{1}{\sqrt{m}}\int_{x' \in \mathcal{X}}  p_{\mirN_{x_i}}(x') dx' \\
    &  \leq \sum_{j=1}^{m} \frac{1}{\sqrt{m}} = \sqrt{m} \\
\end{align*}

\end{proof}

\subsection{Values of $\rho_S$ in-between $1$ and $\sqrt{m}$}
\label{sec:in-between}

We know that the disjointedness factor $\rho_S$ takes the values $1$ and $\sqrt{m}$ in the two extreme cases where the neighborhoods are completely overlapping or disjoint respectively. We also know from Fact~\ref{fact:rho-bound} that the only other values it takes lie in between $1$ and $\sqrt{m}$. When does it take these values?

To get a sense of how these in-between values can be realized, we present a toy example here. Specifically,
we can show that under some simplistic assumptions, 
$\rho_S = \sqrt{m^{{1-k}}}$ (where $0 \leq k \leq 1$) if every neighborhood is just large enough to encompass a $\frac{1}{m^{1-k}}$ fraction of mass of the distribution $D$.

Our main assumption is that $\mirN_{x_i}$  is a uniform distribution over whatever support it covers. Further, to simplify the discussion, assume that $\mathcal{X}$ is a discrete set containing $M$ datapoints in total (think of $M$ as very, very large).

Then, if every neighborhood contains $\frac{1}{m^{1-k}}$ fraction of mass of the distribution $D$, it means it contains $\frac{M}{m^{1-k}}$ points in it. Therefore, since $\mirN_{x_i}$  is a uniform distribution, we have that the probability mass of $\mirN_{x_i}$  on any  point $x'$ in its support is 
$\frac{1}{Mm^{k-1}}$. Plugging this in the definition of $\rho_S$, we get:

\begin{align*}
    \rho_S = \int_{x' \in \mathcal{X}} \sqrt{\frac{1}{m}\sum_{i=1}^{m}(p_{\mirN_{x_i}}(x'))^2} dx' & = \sum_{j=1}^{M} \sqrt{\frac{1}{m}\sum_{i=1}^{m} \left(\Pr_{x'\sim \mirN_{x_i}}\left[x'= x_j \right]\right)^2} \\
    & = \sum_{j=1}^{M} \sqrt{\frac{1}{m}\sum_{i=1}^{m} \mathbb{I}[x_j \in \text{supp}\left(\mirN_{x_i}\right)] \left(\frac{1}{Mm^{k-1}} \right)^2} \\
    & = \sum_{j=1}^{M} \frac{1}{Mm^{k-0.5}} \sqrt{\sum_{i=1}^{m} \mathbb{I}[x_j \in \text{supp}\left(\mirN_{x_i}\right)]} \\
\end{align*}

To further simplify this, we need to compute the innermost summation, which essentially corresponds to the number of mirrored neighborhoods that each point belongs to. For simplicity, let's assume that every point belongs to $n$ neighborhoods. To estimate $n$, observe that for each of the $m$ neighborhoods to contain $\frac{M}{m^{1-k}}$ points, and for each of the $M$ points to be in $n$ neighborhoods, we must have:

\[
M n = m \frac{M}{m^{1-k}}.
\]

Thus, $n = m^k$. Plugging this back in, we get $\rho_S = m^{\frac{1-k}{2}}$.

\section{Piece-wise Global Approximation}
\label{sec:piecewise}

\subsubsection{Generalization bound assuming piecewiseness}
\label{sec:piecewise}
We now discuss the Rademacher complexity of a simpler class of local-approximation functions: a class of piecewise-simple functions $g \in \mathcal{G}$ with $K$ pieces.
In particular, one can show that the complexity of these functions grows with $K$ as $\sqrt{K}$. 

To see why, first let us call the $K$ regions that $g$ is defined over as $R_1, \dots, R_K$. Correspondingly, the original training set $S = \{x_i\}_{i}^{m} $ can be divided into the subsets $S_1 = \{x_{1,i}\}_{i=1}^{m_1}, \dots, S_k= \{x_{K,i}\}_{i=1}^{m_K}$ and the pieces of $g$ are $g_1, \dots, g_K \in \Glocal$ are simple functions. Then, one can split the Rademacher complexity over the whole dataset in terms of these subsets, to get:
    
    \begin{align*}
        \hat{\mathcal{R}}_S(\mathcal G) &= \mathbb{E}_{\sigma}\left[\sup_{g \in \mathcal G} \frac{1}{m} \sum_{i=1}^{m} \sigma_i g(x_i) \right] \\
        &= \mathbb{E}_{\sigma}\left[\sup_{g \in \mathcal G} \sum_{k=1}^{K} \frac{m_k}{m} \sum_{i=1}^{m} \frac{1}{m_k}\sigma_i g_j(x_i) \mathbb{I} \{x_i \in S_j \} \right] \\
        &= \mathbb{E}_{\sigma}\left[\sup_{g \in \mathcal G} \sum_{k=1}^{K} \frac{m_k}{m} \sum_{i=1}^{m_k} \frac{1}{m_k}\sigma_{k,i} g_j(x_{k,i}) \right] \\
        &\leq\sum_{k=1}^{K} \frac{m_k}{m}  \mathbb{E}_{\sigma}\left[\sup_{g_j \in \tilde {\mathcal G}}\frac{1}{m_k}\sigma_{k,i} g_j(x_{k}) \right] \\
        &\leq \sum_{k=1}^{K} \frac{m_k}{m}  \hat{\mathcal{R}}_{S_k}({\mathcal G}_{\text{local}})
    \end{align*}

Now, assuming each $\hat{\mathcal{R}}_{S_k}({\mathcal G}_{\text{local}})$ is  $\mathcal{O}\left(\frac{1}{\sqrt{m_k}}\right)$,  and assuming each subset as the same number of points $m_k = m/K$, the sum in the last expression can be bounded as $\mathcal{O}\left(\sqrt{\frac{K}{m}}\right)$.

\section{Proofs}
\label{sec:proofs}

Below, we present the full statement and proof of Lemma~\ref{lem:rademacher} which bounds the Rademacher complexity of $\mathcal{G}$. The main difference between this statement and the  version in the main paper is that we replace the Rademacher complexity of $\Glocal$ with a slightly more carefully defined version of it defined below:

\begin{equation}
\hat{\mathcal{R}}^*_S(\Glocal) := \max_{i \leq m} \max_{T \subseteq S, |T|=i} \hat{\mathcal{R}}_{T}(\Glocal) \sqrt{\frac{i}{m}}
\label{eq:corrected-rademacher}
\end{equation}

This quantity is essentially a bound on the empirical Rademacher complexity of $\Glocal$ on all possible subsets of $S$, with an appropriate scaling factor. 

We note that although this quantity is technically larger than the original quantity namely $\hat{\mathcal{R}}_S(\Glocal)$, for all practical purposes, it is reasonable to think of $\hat{\mathcal{R}}^*_S(\Glocal)$ as being identical to $\hat{\mathcal{R}}_S(\Glocal)$ modulo some constant factor. For example, if we have that for all $h \in \Glocal$, $h(x) = w \cdot x$ where $\| w\|_2 \leq \alpha$, then one would typically bound $\hat{\mathcal{R}}_S(\Glocal)$ by $O\left(\frac{\alpha \sqrt{{\sum_{i=1}^{m} \| x_i\|_2^2}/{m}}}{\sqrt{m}}\right)$. The bound on $\hat{\mathcal{R}}^*_S(\Glocal)$ however would resolve to $O\left(\frac{\alpha \sqrt{\max_{i \leq m} \| x_i\|_2^2}}{\sqrt{m}}\right)$. Now, as long as we assume that $\| x_i\|$ are all bounded by some constant, both these bounds are asymptotically the same, and have the same $1/\sqrt{m}$ dependence on $m$. Additionally, we also remark that that it is possible to write our results in terms of tighter definitions of $\hat{\mathcal{R}}^*_S(\Glocal)$, however our statements read much cleaner with the above definition. 

\begin{lemma} \textbf{(full, precise statement of Lemma~\ref{lem:rademacher})}
\label{lem:rademacher-full}
Let $L(\cdot,y')$ be a $c$-Lipschitz function w.r.t. $y'$ in that for all $y_1, y_2 \in [-B,B]$, $|L(y_1,y') - L(y_2, y')| \leq c |y_1-y_2|$. Let $S = \{(x_1,y_1), \hdots, (x_m,y_m) \} \in \mathcal{X}^m$. Then, the empirical Rademacher complexity of $\mathcal{G}$ under the loss function $L$ is defined and bounded as:
  \[
  \hat{\mathcal{R}}_S(L \circ \mathcal{G}) := \mathbb{E}_{\vec{\sigma}}\left[ \sup_{g \in {\mathcal{G}}} \frac{1}{m} \sum_{i}^{m} \sigma_i \mathbb{E}_{x' \sim \mirN_{x_i}}[L(g_{x'}(x_i),y_i)] \right] \leq  c \rho_S   (\ln m +1) \cdot \hat{\mathcal{R}}^*_{S}(\Glocal).
  \]
  where recall that $\rho_S := \int_{x' \in \mathcal{X}} \sqrt{ \frac{\sum_{j=1}^{m} (p_{\mirN_{x_i}}(x'))^2}{m}} dx'$ is the disjointedness factor.
\end{lemma}

Our high level proof idea is to first construct a distribution $\tilde{D}$ over $\mathcal{X}$ in a way that each the inner expectations over $\mirN_{x_i}$ (for each $i$) can be rewritten as an expectation over $x' \sim \tilde{D}$. This removes the dependence on $i$ from this expectation, which then allows us to pull this expectation all the way out. This further allows us to take each $x'$ and compute a Rademacher complexity corresponding to the loss of $g_{x'}$, and then finally average that complexity over $x' \sim \tilde{D}$. 

\begin{proof}
We begin by noting that the inner expectations in the Rademacher complexity are over $m$ unique distributions $\mirN_{x_i}$. our first step is to rewrite these expectations in a way that they all apply on the same distribution. Let us call this distribution $\tilde{D}$ and define what it is later. As long as $\tilde{D}$ has a support that contains the support of the above $m$ distributions, we can write:
\begin{align*}
\hat{\mathcal{R}}_{S}(L\circ \mathcal{G}) &= \mathbb{E}_{\vec{\sigma}}\left[ \sup_{g \in {\mathcal{G}}} \frac{1}{m} \sum_{i}^{m} \sigma_i \mathbb{E}_{x' \sim \tilde{D}}\left[L(g_{x'}(x_i),y_i) \frac{p_{\mirN_{x_i}}(x')}{p_{\tilde{D}}(x')}\right] \right]\\
\intertext{this allows us to pull the inner expectation in front of the supremum (which makes this an inequality now):}
& \leq \mathbb{E}_{\vec{\sigma}}\left[  \mathbb{E}_{x' \sim \tilde{D}}\left[ \sup_{g \in {\mathcal{G}}} \frac{1}{m} \sum_{i}^{m} \sigma_i L(g_{x'}(x_i),y_i) \frac{p_{\mirN_{x_i}}(x')}{p_{\tilde{D}}(x')}\right] \right]\\
\intertext{which further allows us rewrite the supremum to be over $\Glocal$ instead of ${\mathcal{G}}$:}
& \leq \mathbb{E}_{\vec{\sigma}}\left[  \mathbb{E}_{x' \sim \tilde{D}}\left[ \sup_{h \in \Glocal} \frac{1}{m} \sum_{i}^{m} \sigma_i L(h(x_i),y_i) \frac{p_{\mirN_{x_i}}(x')}{p_{\tilde{D}}(x')}\right] \right]\\
\intertext{ next, let us simply interchange the two outer expectations and rewrite it as:}
& \leq \mathbb{E}_{x' \sim \tilde{D}} \left[ \mathbb{E}_{\vec{\sigma}} \left[ \sup_{h \in \Glocal} \frac{1}{m} \sum_{i}^{m} \sigma_i L(h(x_i),y_i) \frac{p_{\mirN_{x_i}}(x')}{p_{\tilde{D}}(x')}\right] \right].\\
\end{align*}

What we now have is an inner expectation which boils down to an empirical Rademacher complexity for a fixed $x'$, and an outer expectation that averages this over $x' \sim \tilde{D}$. For the rest of the discussion, we will fix $x'$ and focus on bounding the inner term. For convenience, let us define $w_i := \frac{p_{\mirN_{x_i}}(x')}{p_{\tilde{D}}(x')}$.
 Without loss of generality,  assume that $w_1 \leq w_2 \leq \hdots \leq w_m$. Also define $w_0 := 0$. We then begin by expanding $w_i$ into a telescopic summation:


\begin{align*}
  \mathbb{E}_{\vec{\sigma}} \left[ \sup_{h \in \Glocal} \frac{1}{m} \sum_{i=1}^{m} \sigma_i L(h(x_i),y_i) w_i \right] & = 
  \mathbb{E}_{\vec{\sigma}} \left[ \sup_{h \in \Glocal} \frac{1}{m} \sum_{i=1}^{m} 
  \sigma_i L(h(x_i),y_i) \sum_{j=1}^{i} (w_j-w_{j-1}) \right]\\
  \intertext{then, we interchange the two summations while adjusting their limits appropriately:} 
  &  = \mathbb{E}_{\vec{\sigma}}\left[  \sup_{h \in \Glocal} \frac{1}{m} \sum_{j=1}^{m} \sum_{ i=j}^{m} \sigma_i L(h(x_i),y_i) (w_j - w_{j-1}) \right]\\
  \intertext{and we pull out the outer summation in front of the supremum and expectation, making it an upper bound:}
&  \leq \sum_{j=1}^{m} \mathbb{E}_{\vec{\sigma}}\left[  \sup_{h \in \Glocal} \frac{1}{m} \sum_{i=j}^{m} \sigma_i L(h(x_i),y_i) (w_j - w_{j-1}) \right].\
\end{align*}

Intuitively, the above steps have executed the following idea. The Rademacher complexity on the LHS can be thought of as involving a dataset with weights $w_1, w_2, \hdots, w_m$ given to the losses on each of the $m$ datapoints. We then imagine decomposing this ``weighted'' dataset into multiple weighted datasets while ensuring that the weights summed across these datasets equal $w_1, w_2, \hdots, w_m$ on the respective datapoints. Then, we could compute the Rademacher complexity for each of these datasets, and then sum them up to get an upper bound on the complexity corresponding to the original dataset. 

The way we decomposed the datasets is as follows: first we extract a $w_1$ weight out of all the $m$ data points (which is possible since it's the smallest weight), giving rise to a dataset of $m$ points all with equal weights $w_1$. What remains is a dataset with weights $0, w_2 - w_1, w_3 - w_1, \hdots, w_m - w_1$. From this, we'll extract a $w_2 - w_1$ weight out of all but the first data point to create a dataset of $m-1$ datapoints all equally weighted as $w_{2} - w_1$. By proceeding similarly, we can generate $m$ such datasets of cardinality $m$, $m-1$, $\hdots$, $1$ respectively, such that all datasets have equally weighted points, and the weights follow the sequence $w_1-w_0, w_2-w_1$, $\hdots$ and so on. As stated before, we will eventually sum up Rademacher complexity terms computed with respect to each of these datasets.

Now, we continue simplifying the above term by pulling out $(w_j - w_{j-1})$ since it is only a constant:
 \begin{align*}
   \mathbb{E}_{\vec{\sigma}} \left[ \sup_{h \in \Glocal} \frac{1}{m} \sum_{i=1}^{m} \sigma_i L(h(x_i),y_i) w_i \right] &  \leq \sum_{j=1}^{m} (w_j - w_{j-1}) \mathbb{E}_{\vec{\sigma}}\left[  \sup_{h \in \Glocal} \frac{1}{m} \sum_{i=j}^{m} \sigma_i L(h(x_i),y_i)  \right]\\
\intertext{next, we apply the standard contraction lemma (Lemma~\ref{lem:contraction}) to make use of the fact $h(x_i)$ is composed with a $c$-Lipschitz function to get:}
&  \leq c \sum_{j=1}^{m} (w_j - w_{j-1}) \mathbb{E}_{\vec{\sigma}}\left[  \sup_{h \in \Glocal} \frac{1}{m} \sum_{i=j}^{m} \sigma_i h(x_i)  \right]\\
\intertext{using $S_{j:m}$ to denote the datapoints indexed from $j$ to $m$, we can rewrite this in short as:}
&  \leq c \sum_{j=1}^{m} (w_j - w_{j-1}) \frac{m+1-j}{m}\hat{\mathcal{R}}_{S_{j:m}}(\Glocal)\\
\intertext{and finally, we make use of the definition of $\mathcal{R}_S^*(\Glocal)$ in Equation~\ref{eq:corrected-rademacher} to get:}
&  \leq c \sum_{j=1}^{m} (w_j - w_{j-1})\frac{\sqrt{m+1-j}}{\sqrt{m}}\hat{\mathcal{R}}^*_{S}(\Glocal) .\\
 \end{align*}

What remains now is to simplify the summation over $w$'s. To do this, we rearrange the telescopic summation as follows:
\begin{align*}
  \sum_{j=1}^{m} (w_j - w_{j-1}) \sqrt{m+1-j} &=  \sum_{j=1}^{m} w_j (\sqrt{m+1-j}-\sqrt{m-j}) \\
  & =  \sum_{j=1}^{m} w_j \cdot \frac{1}{\sqrt{m+1-j} + \sqrt{m-j}} \\
  & \leq \sum_{j=1}^{m} w_j \frac{1}{\sqrt{m+1-j}} \\
  & \leq  \sqrt{\sum_{j=1}^{m} w_j^2} \cdot \sqrt{\sum_{j=1}^{m}\frac{1}{j}} \\
  & \leq   \sqrt{\sum_{j=1}^{m} w_j^2} \cdot  (\ln m +1)
\end{align*}

Note that in the penultimate step we've used the Cauchy-Schwartz inequality and in the last step, we have made use of the standard logarithmic upper bound on the $m$-th harmonic number. Plugging this back on the Rademacher complexity bound, we get:

\begin{align*}
\hat{\mathcal{R}}_{S}(L\circ \mathcal{G}) &\leq
 \mathbb{E}_{x' \sim \tilde{D}} \left[ c \sqrt{\sum_{j=1}^{m}w^2_{j}} \cdot  (\ln m +1) \cdot \frac{\hat{\mathcal{R}}^*_{S}(\Glocal)}{\sqrt{m}}\right]\\
 \intertext{plugging in the values of $w_j$, we get:}
 &\leq
 \mathbb{E}_{x' \sim \tilde{D}} \left[ c  \sqrt{ \frac{\sum_{j=1}^{m} (p_{\mirN_{x_i}}(x'))^2}{(p_{\tilde{D}}(x'))^2}} \cdot  (\ln m +1) \cdot \frac{\hat{\mathcal{R}}^*_{S}(\Glocal)}{\sqrt{m}}\right].\\
 & \leq c \mathbb{E}_{x' \sim \tilde{D}} \left[ \sqrt{ \frac{\sum_{j=1}^{m} \frac{(p_{\mirN_{x_i}}(x'))^2}{m}}{(p_{\tilde{D}}(x'))^2}}\right]   (\ln m +1) \cdot \hat{\mathcal{R}}^*_{S}(\Glocal) .\\
\end{align*}

Now we finally set $\tilde{D}$ such that $p_{\tilde{D}}(x') = \frac{\sqrt{\sum_{j=1}^{m} \frac{(p_{\mirN_{x_i}}(x'))^2}{m}}}{\rho_S}$ where $\rho_S$ is a normalization constant such that
$\rho_S = \int_{x' \in \mathcal{X}} \sqrt{\sum_{j=1}^{m} \frac{(p_{\mirN_{x_i}}(x'))^2}{m}} dx'$.
Then, the above term would simplify as:
\begin{align*}
\hat{\mathcal{R}}_{S}(L\circ \mathcal{G}) & \leq c \mathbb{E}_{x' \sim \tilde{D}} \left[\rho_S\right]   (\ln m +1) \cdot \hat{\mathcal{R}}^*_{S}(\Glocal) \\ 
& \leq c \rho_S   (\ln m +1) \cdot \hat{\mathcal{R}}^*_{S}(\Glocal).
\end{align*}
\end{proof}

Next, we state and prove the full version of Theorem~\ref{thm:performance-generalization} which provided a generalization guarantee for the test error of $f$ in terms of its local interpretability.

\begin{theorem} 
\label{thm:performance-generalization-full}
\textbf{(full, precise version of Theorem~\ref{thm:performance-generalization})}
With probability over $1-\delta$ over the draws of $S = \{(x_1,y_1), \hdots, (x_m,y_m) \} \sim D^m$, for all $f \in \mathcal{F}$ and for all $g \in \mathcal{G}$, we have (ignoring $\ln 1/\delta$ factors):
\begin{align*}
     \mathbb{E}_{(x,y) \sim D}[(f(x)-y)^2] & \leq \frac{4}{m}\sum_{i=1}^{m} (f(x_i)-y_i)^2 + 2 \underbrace{\mathbb{E}_{x \sim D}[\mathbb{E}_{x' \sim \mirN_x} \left[ (f(x) - g_{x'}(x))^2\right]]}_{\mirNF(f,g)} \\
     & + \frac{4}{m}\sum_{i=1}^{m} \underbrace{\mathbb{E}_{x' \sim \mirN_x} \left[ (f(x_i) - g_{x'}(x_i))^2\right]}_{\mirNF(f,g,x_i)} +
     16B\rho_S \hat{\mathcal{R}}^*_S(\Glocal) (\ln m+1) \\
     & + 2\sqrt{\frac{\ln 1/\delta}{m}},
\end{align*}
where $\rho_S$ denotes the \textbf{disjointedness factor} defined as $\rho_S := \int_{x' \in \mathcal{X}} \sqrt{\frac{1}{m}\sum_{i=1}^{m}(p_{\mirN_{x_i}}(x'))^2} dx'$
and $\hat{\mathcal{R}}^*_S(\Glocal)$ is defined in Equation~\ref{eq:corrected-rademacher}.
\end{theorem}

\begin{proof} 
First, we split the test error into two terms by introducing the $g$ function as follows:

\begin{align*}
\mathbb{E}_{(x,y) \sim D}[(f(x)-y)^2] & = \mathbb{E}_{(x,y) \sim D}[\mathbb{E}_{x' \sim \mirN_x}[(f(x)-y)^2]] \\
& \leq  2 \left(\mathbb{E}_{x \sim D}[\mathbb{E}_{x' \sim \mirN_x}[(f(x)-g_{x'}(x))^2]] 
+  \mathbb{E}_{x \sim D}[\mathbb{E}_{x' \sim \mirN_x}[(g_{x'}(x) - y)^2]] \right) \numberthis \label{eq:1}
\end{align*}

In the first step, we have introduced a dummy expectation over $x'$, and in the next step, we have used the following inequality: for any $a,b,c \in \mathbb{R}$, $(a-b)^2 \leq (|a-c| + |c-b|)^2 \leq 2(|a-c|^2 + |c-b|^2)$ (the first inequality in this line is the triangle inequality and the second inequality is the root mean square inequality). 

The first term on the RHS above is $\mirNF(f,g)$. To simplify the second term, we first apply a generalization bound based on Rademacher complexity. Specifically, we have that w.h.p $1-\delta$ over the draws of $S$, for all $g \in \mathcal{G}$,

\begin{align*}
    \mathbb{E}_{x \sim D}[\mathbb{E}_{x' \sim \mirN_x}[(g_{x'}(x) - y)^2]]  \leq \frac{1}{m}\sum_{i=1}^{m} \mathbb{E}_{x' \sim \mirN_{x_i}}[(g_{x'}(x_i) - y_i)^2] + 2\hat{\mathcal{R}}_{S}(\mathcal{G}) + \sqrt{\frac{\ln 1/\delta}{m}} \numberthis \label{eq:2}
\end{align*}
Now, $\hat{\mathcal{R}}_{S}(\mathcal{G})$ can be bounded using Lemma~\ref{lem:rademacher} under Lipschitzness of the squared error loss. Specifically, we have that for $h, h' \in \Glocal$, and for all $y \in [-B,B]$, $|(h(x) - y)^2 - (h'(x) - y)^2| \leq 4B|h(x) - h'(x)|$, since all of $h(x), h'(x)$ and $y$ lie in $[-B,B]$. Therefore, from Lemma~\ref{lem:rademacher} we have that:

\begin{align*}\hat{\mathcal{R}}_{S}(\mathcal{G}) \leq 4B(\ln m +1) \rho_S \hat{\mathcal{R}}^*_{S}(\Glocal). \numberthis \label{eq:3}
\end{align*}

The only term that remains to be bounded is the first term on the RHS. This can bounded again using the inequality that for any $a,b,c \in \mathbb{R}$, $(a-b)^2 \leq (|a-c| + |c-b|)^2 \leq 2(|a-c|^2 + |c-b|^2)$:

\begin{align*}
    \frac{1}{m}\sum_{i=1}^{m} \mathbb{E}_{x' \sim \mirN_{x_i}}[(g_{x'}(x_i) - y_i)^2)] \leq   \frac{2}{m}\sum_{i=1}^{m} \mathbb{E}_{x' \sim \mirN_{x_i}}[(g_{x'}(x_i) - f(x_i))^2] + \frac{2}{m}\sum_{i=1}^{m} (f(x_i) - y_i)^2 \numberthis \label{eq:4}
\end{align*}

By combining the above three chains of inequalities, we get the final bound.
\end{proof}

Below, we present an alternative version of Theorem~\ref{thm:performance-generalization} where the generalization bound does not involve the test $\mirNF$ and hence does not require any unlabeled data from $D$; however the bound is not on the test error of $f$ but the test error of $g$.

\begin{theorem} 
\label{thm:performance-generalization-full-2}
\textbf{(an alternative version of Theorem~\ref{thm:performance-generalization})}
With probability over $1-\delta$ over the draws of $S = \{(x_1,y_1), \hdots, (x_m,y_m) \} \sim D^m$, for all $f \in \mathcal{F}$ and for all $g \in \mathcal{G}$, we have:
\begin{align*}
     \mathbb{E}_{(x,y) \sim D}[\mathbb{E}_{x' \sim \mirN_{x}}[(g_{x'}(x)-y)^2]] & \leq \frac{2}{m}\sum_{i=1}^{m} (f(x_i)-y_i)^2 + \frac{2}{m}\sum_{i=1}^{m} \underbrace{\mathbb{E}_{x' \sim \mirN_x} \left[ (f(x_i) - g_{x'}(x_i))^2\right]}_{\mirNF(f,g,x_i)} \\
     &+ 
     8B\rho_S \hat{\mathcal{R}}_S(\Glocal) (\ln m+1) + \sqrt{\frac{\ln 1/\delta}{m}}.
\end{align*}
\end{theorem}

\begin{proof}
The proof follows directly from the proof of Theorem~\ref{thm:performance-generalization-full-2} starting from Equation~\ref{eq:2}.
\end{proof}

We now state and prove the full version of Theorem~\ref{thm:explanation-generalization} which provided a generalization guarantee for the quality of explanations.
\begin{theorem}
\label{thm:explanation-generalization-full} (full, precise statement of Theorem~\ref{thm:explanation-generalization})
For a fixed function $f$, with high probability $1-\delta$ over the draws of $S \sim D^m$, for all $g \in \mathcal{G}$, we have:
\begin{align*}
\underbrace{\mathbb{E}_{x \sim D} \left[ \mathbb{E}_{x' \sim \mirN_{x}} \left[ (f(x)-g_{x'}(x))^2\right]\right]}_{\text{test } \mirNF \text{ i.e., } \mirNF(f,g)} & \leq \underbrace{\frac{1}{m}\sum_{i=1}^{m} \mathbb{E}_{x' \sim \mirN_{x}} \left[ (f(x_i)-g_{x'}(x_i))^2\right]}_{ \text{train } \mirNF} \\
& + 8B\rho_S \mathcal{R}_S(\Glocal) \ln m   + \sqrt{\frac{\ln 1/\delta}{m}}.
\end{align*}
where $\hat{\mathcal{R}}^*_S(\Glocal)$ is defined in Equation~\ref{eq:corrected-rademacher}.
\end{theorem}

\begin{proof}
For this result, we need to think of $f$ as a fixed labeling function since it is independent of the dataset $S$ that is used to train $g$. Then, one can apply a standard Rademacher complexity bound and invoke Lemma~\ref{lem:rademacher} to get the final result (as invoked in Equation~\ref{eq:3}). 
\end{proof}

Below, we state the standard contraction lemma for Rademacher complexity. The lemma states that composing a function class with a $c$-Lipschitz function can scale up its Rademacher complexity by a multiplicative factor of atmost $c$.
\begin{lemma}
\label{lem:contraction}
\textbf{(Contraction lemma)}
For each $i = 1, 2, \hdots, m$, let $\phi_i: \mathbb{R} \to \mathbb{R}$ be a $c$-Lipschitz function in that for all $t, t' \in \mathcal{B} \subseteq \mathbb{R}$, $|\phi_i(t) - \phi_i(t')| \leq |t-t'|$. Then, for any class $\mathcal{H}$ of functions $h: \mathbb{R} \to \mathcal{B}$, we have:

\[
\mathbb{E}_{\vec{\sigma}}\left[ \sum_{i=1}^{m} \sigma_i \phi_i(h(x_i)) \right] \leq 
c \mathbb{E}_{\vec{\sigma}}\left[ \sum_{i=1}^{m} \sigma_i (h(x_i) \right].
\]
\end{lemma}

\section{Experiment Details}
\label{experiments-appendix}
\subsection{Procedure for calculating $\rho_S$}

As a reminder, we define $\rho_S$ to be an integral over $\mathcal{X}$, which is not trivial to evaluate in practice, especially in higher dimensions. 
 \[
\rho_S = \int_{x' \in \mathcal{X}} \sqrt{\frac{1}{m}\sum_{i=1}^{m}(p_{\mirN_{x_i}}(x'))^2} dx'
\]
\noindent Common numerical integration techniques usually incur significant computational costs due to the dimension of $x$. Though a variety of methods exist, one can intuit this blow-up by considering the naive approach of simply constructing a Riemann sum across a rectangular meshgrid of points in $\mathcal{X}$. If one wants to create a grid of $c$ points per dimension, then $c^d$ points (and thus evaluations of the integrand) must be processed. \\

\noindent Instead, we can apply Monte-Carlo Integration to evaluate $\rho_S$. As we will see, a key feature of this approach is that error will \textit{not} scale with data dimension and can be bounded probabilistically via a Hoeffding bound. Currently, the integral does not look like an expectation so we must introduce a dummy distribution $q(x')$ as follows
$$
\rho_S  = \int_{x' \in \mathcal{X}} \frac{\sqrt{\frac{1}{m}\sum_{i=1}^{m}(p_{\mirN_{x_i}}(x'))^2}}{q(x')}q(x') dx'= \mathbb{E}_{x' \sim q} \left[ \frac{\sqrt{\frac{1}{m}\sum_{i=1}^{m}(p_{\mirN_{x_i}}(x'))^2}}{q(x')} \right]$$
\noindent Now, we can estimate $\rho_S$ with $n$ independent samples from $q$.

$$\hat{\rho}_{S,n} = \frac{1}{n} \sum_{j=1}^n \frac{\sqrt{\frac{1}{m}\sum_{i=1}^{m}(p_{\mirN_{x_i}}(x'_j))^2}}{q(x'_j)}$$

\noindent This is an unbiased estimate of $\rho_S$, but that in itself is not sufficient. This is only a feasible approach if we can choose $q$ such that (1) we can actually sample from it, (2) we can calculate $q(x')$ for arbitrary $x'$ and (3) we can control the variance of  $\frac{\sqrt{\frac{1}{m}\sum_{i=1}^{m}(p_{\mirN_{x_i}}(x'))^2}}{q(x')}$. \\

\noindent It can be shown by choosing $q$ to be a uniform mixture of the $m$ training set neighborhoods, we can satisfy all 3 properties. (1) and (2) are dependent on those same properties being satisfied by $\mirN_x$. If $\mirN_x$ can be sampled from, the mixture over $m$ such distributions can obviously be sampled from. The same goes for calculating the density, which in this case is:
$$ q(x')  = \sum_{i=1}^m \frac{1}{m} \cdot p_{\mirN_{x_i}}(x') = \frac{1}{m} \sum_{i=1}^m  p_{\mirN_{x_i}}(x') $$
\noindent We observe that (3) can also be shown because we can upper and lower bound the quantity in question. To show this, we first re-write it as

\begin{align*}
\frac{\sqrt{\frac{1}{m}\sum_{i=1}^{m}(p_{\mirN_{x_i}}(x'))^2}}{q(x')} &= \frac{\sqrt{\frac{1}{m}\sum_{i=1}^{m}(p_{\mirN_{x_i}}(x'))^2}}{\frac{1}{m} \sum_{i=1}^m \cdot p_{\mirN_{x_i}}(x')} \\
&= \sqrt{m}\frac{\sqrt{\sum_{i=1}^{m}(p_{\mirN_{x_i}}(x'))^2}}{\sum_{i=1}^m \cdot p_{\mirN_{x_i}}(x')} \\ 
&= \sqrt{m}\frac{||p_{S}(x')||_2}{||p_{S}(x')||_1}
\end{align*}

\noindent where $p_S(x')$ is a $m$-dimensional vector of densities each evaluated at $x'$ (i.e. one for each of the $m$ training points). Since $||x||_2 \leq ||x||_1 \leq \sqrt{m} ||x||_2$, the upper and lower bounds for this quantity are $\sqrt{m}$ and $1$ respectively. Thus we can bound the variance of this quantity by $\frac{1}{4}(\sqrt{m}-1)^2 \leq \frac{m}{4} $ and $\text{Var}(\hat{\rho}_{S,n}) \leq \frac{m}{4n}$. This does not scale with dimension but only the number of training points!

\noindent To be even more concrete, for a given $m$ and $n$, we can now apply a Hoeffding bound to control the error. 
$$ \mathbb{P}(|\hat{\rho}_{S,n} - \rho_S| > t) \leq 2e^{\frac{-2nt^2}{m}}$$

In our experiments we choose $n$ to be $10m$, meaning that the probability that $\rho_S$ is off by more than 0.5 is capped at about $1\%$ (recall that $\rho_S$ scales from $[1, \sqrt{m}]$. 

\subsection{Full set of results}

\begin{figure}[h]
    \centering \includegraphics[width=0.3\textwidth]{figures/autompgs.png} \includegraphics[width=0.3\textwidth]{figures/housing.png} \includegraphics[width=0.3\textwidth]{figures/wine.png} \includegraphics[width=0.3\textwidth]{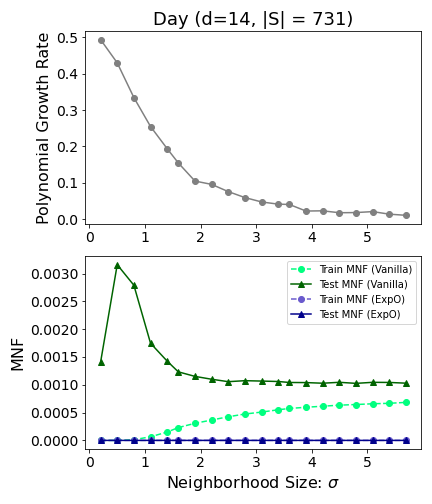} \includegraphics[width=0.3\textwidth]{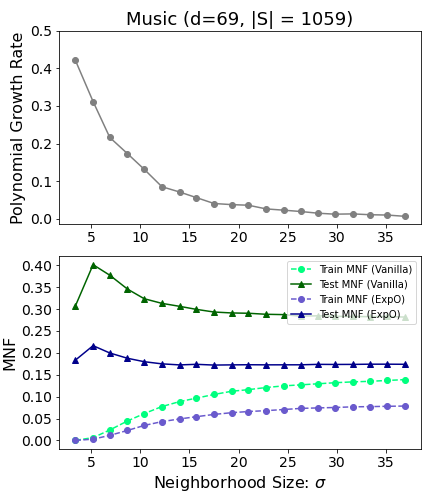} 
    \caption{ Approximate exponent of $\rho_S$'s polynomial growth rate (top) and train/test $\mirNF$ (below) plotted for various neighborhood widths across several UCI datasets. 
    }
    \label{fig:experiments_full}
\end{figure} 

\


%% file: iclr2021_conference.bbl
\begin{thebibliography}{19}
\providecommand{\natexlab}[1]{#1}
\providecommand{\url}[1]{\texttt{#1}}
\expandafter\ifx\csname urlstyle\endcsname\relax
  \providecommand{\doi}[1]{doi: #1}\else
  \providecommand{\doi}{doi: \begingroup \urlstyle{rm}\Url}\fi

\bibitem[Arlot \& Genuer(2014)Arlot and Genuer]{ArlotGenuer2014}
Sylvain Arlot and Robin Genuer.
\newblock Analysis of purely random forests bias, 2014.

\bibitem[Arora et~al.(2018)Arora, Ge, Neyshabur, and Zhang]{Arora2018}
Sanjeev Arora, Rong Ge, Behnam Neyshabur, and Yi~Zhang.
\newblock Stronger generalization bounds for deep nets via a compression
  approach.
\newblock volume~80 of \emph{Proceedings of Machine Learning Research}, pp.\
  254--263, Stockholmsmässan, Stockholm Sweden, 10--15 Jul 2018. PMLR.

\bibitem[Dua \& Graff(2017)Dua and Graff]{Dua:2019}
Dheeru Dua and Casey Graff.
\newblock {UCI} machine learning repository, 2017.
\newblock URL \url{http://archive.ics.uci.edu/ml}.

\bibitem[Dziugaite \& Roy(2017)Dziugaite and Roy]{Dziugaite2017}
Gintare~Karolina Dziugaite and Daniel~M. Roy.
\newblock Computing nonvacuous generalization bounds for deep (stochastic)
  neural networks with many more parameters than training data.
\newblock In \emph{Proceedings of the Thirty-Third Conference on Uncertainty in
  Artificial Intelligence, UAI 2016}. 2017.

\bibitem[Fan(1993)]{Fan1993}
Jianqing Fan.
\newblock Local linear regression smoothers and their minimax efficiencies.
\newblock \emph{The Annals of Statistics}, 21, 03 1993.
\newblock \doi{10.1214/aos/1176349022}.

\bibitem[Fan \& Gijbels(1996)Fan and Gijbels]{FanGijbels1996}
{Jianqing} Fan and {Irène} Gijbels.
\newblock \emph{Local polynomial modelling and its applications}.
\newblock Number~66 in Monographs on statistics and applied probability series.
  Chapman \& Hall, London [u.a.], 1996.
\newblock ISBN 0412983214.

\bibitem[Langford \& Caruana(2002)Langford and Caruana]{Langford2002}
John Langford and Rich Caruana.
\newblock (\uppercase{N}ot) bounding the true error.
\newblock In T.~G. Dietterich, S.~Becker, and Z.~Ghahramani (eds.),
  \emph{Advances in Neural Information Processing Systems 14}, pp.\  809--816.
  MIT Press, 2002.

\bibitem[Langford \& Shawe-Taylor(2003)Langford and Shawe-Taylor]{Langford2003}
John Langford and John Shawe-Taylor.
\newblock Pac-bayes \& margins.
\newblock In S.~Becker, S.~Thrun, and K.~Obermayer (eds.), \emph{Advances in
  Neural Information Processing Systems 15}, pp.\  439--446. MIT Press, 2003.

\bibitem[McAllester(2003)]{McAllester2003}
David McAllester.
\newblock Simplified pac-bayesian margin bounds.
\newblock In Bernhard Sch{\"o}lkopf and Manfred~K. Warmuth (eds.),
  \emph{Learning Theory and Kernel Machines}, pp.\  203--215, Berlin,
  Heidelberg, 2003. Springer Berlin Heidelberg.
\newblock ISBN 978-3-540-45167-9.

\bibitem[McAllester(1998)]{McAllester1998}
David~A McAllester.
\newblock Some pac-bayesian theorems.
\newblock In \emph{11th annual conference on Computational learning theory},
  1998.

\bibitem[Nagarajan \& Kolter(2019)Nagarajan and Kolter]{Nagarajan2019}
Vaishnavh Nagarajan and J.~Zico Kolter.
\newblock Uniform convergence may be unable to explain generalization in deep
  learning.
\newblock In \emph{Advances in Neural Information Processing Systems 32}, pp.\
  11615--11626. Curran Associates, Inc., 2019.

\bibitem[Neyshabur et~al.(2014)Neyshabur, Tomioka, and Srebro]{Neyshabur2014}
Behnam Neyshabur, Ryota Tomioka, and Nathan Srebro.
\newblock In search of the real inductive bias: On the role of implicit
  regularization in deep learning, 2014.
\newblock URL \url{https://arxiv.org/abs/1412.6614}.

\bibitem[Plumb et~al.(2018)Plumb, Molitor, and Talwalkar]{Plumb2018MAPLE}
Gregory Plumb, Denali Molitor, and Ameet~S Talwalkar.
\newblock Model agnostic supervised local explanations.
\newblock In S.~Bengio, H.~Wallach, H.~Larochelle, K.~Grauman, N.~Cesa-Bianchi,
  and R.~Garnett (eds.), \emph{Advances in Neural Information Processing
  Systems 31}, pp.\  2515--2524. Curran Associates, Inc., 2018.

\bibitem[Plumb et~al.(2020)Plumb, Al-Shedivat, Cabrera, Perer, Xing, and
  Talwalkar]{Plumb2020ExpO}
Gregory Plumb, Maruan Al-Shedivat, Angel~Alexander Cabrera, Adam Perer, Eric
  Xing, and Ameet Talwalkar.
\newblock Regularizing black-box models for improved interpretability, 2020.
\newblock URL \url{https://arxiv.org/abs/1902.06787}.

\bibitem[Ribeiro et~al.(2016)Ribeiro, Singh, and Guestrin]{Ribeiro2016}
Marco~Tulio Ribeiro, Sameer Singh, and Carlos Guestrin.
\newblock "\uppercase{W}hy should i trust you?": Explaining the predictions of
  any classifier.
\newblock In \emph{Proceedings of the 22nd ACM SIGKDD International Conference
  on Knowledge Discovery and Data Mining, pages 1135–1144. ACM, 2016}. 2016.

\bibitem[Ribeiro et~al.(2018)Ribeiro, Singh, and Guestrin]{Ribeiro2018}
Marco~Tulio Ribeiro, Sameer Singh, and Carlos Guestrin.
\newblock Anchors: High-precision model-agnostic explanations.
\newblock In \emph{AAAI Conference on Artificial Intelligence}. 2018.
\newblock URL
  \url{https://www.aaai.org/ocs/index.php/AAAI/AAAI18/paper/view/16982}.

\bibitem[Semenova et~al.(2019)Semenova, Rudin, and Parr]{Semenova2019Rashomon}
Lesia Semenova, Cynthia Rudin, and Ronald Parr.
\newblock A study in rashomon curves and volumes: A new perspective on
  generalization and model simplicity in machine learning, 2019.
\newblock URL \url{https://arxiv.org/abs/1908.01755}.

\bibitem[Yoon et~al.(2019)Yoon, Arik, and Pfister]{Yoon2019}
Jinsung Yoon, Sercan~O. Arik, and Tomas Pfister.
\newblock \uppercase{RL-LIM}: Reinforcement learning-based locally
  interpretable modeling, 2019.
\newblock URL \url{https://arxiv.org/abs/1909.12367}.

\bibitem[Zhang et~al.(2017)Zhang, Bengio, Hardt, Recht, and Vinyals]{Zhang2017}
Chiyuan Zhang, Samy Bengio, Moritz Hardt, Benjamin Recht, and Oriol Vinyals.
\newblock Understanding deep learning requires rethinking generalization.
\newblock In \emph{5th International Conference on Learning Representations,
  {ICLR} 2017, Toulon, France, April 24-26, 2017, Conference Track
  Proceedings}, 2017.

\end{thebibliography}
